\documentclass[11pt]{article}

\usepackage[margin=1in]{geometry}
\usepackage[T1]{fontenc}
\usepackage[utf8]{inputenc}
\usepackage{amsmath,amssymb}
\usepackage{amsthm}
\usepackage{array}
\usepackage{booktabs}
\usepackage{caption}
\usepackage{subcaption}
\usepackage{float}
\usepackage{graphicx}
\usepackage{hyperref}
\usepackage{microtype}
\usepackage{placeins}
\usepackage{tabularx}
\usepackage{tikz}
\usepackage{standalone}
\usepackage{xcolor}
\usetikzlibrary{arrows.meta,calc,positioning,fit,backgrounds,decorations.pathmorphing,shapes.geometric}

\graphicspath{{./}}
\hypersetup{
  colorlinks=true,
  linkcolor=blue!60!black,
  citecolor=blue!60!black,
  urlcolor=blue!60!black
}

\newcommand{\pjrope}{PJ-RoPE}

\definecolor{pjblue}{RGB}{23,84,190}
\definecolor{pjpurple}{RGB}{92,43,155}
\definecolor{pjgreen}{RGB}{21,110,35}
\definecolor{pjorange}{RGB}{220,86,0}
\definecolor{pjteal}{RGB}{0,125,132}
\definecolor{pjred}{RGB}{210,28,28}
\definecolor{pjgray}{RGB}{92,92,92}
\newtheorem{definition}{Definition}
\newtheorem{proposition}{Proposition}
\newtheorem{lemma}{Lemma}[section]

\title{PJ-RoPE: A Fourier--Jet--Affine Position Space for Relative Attention}
\author{Yaobo Zhang\\
School of Physics, Ningxia University\\
Yinchuan, 750021, China\\
\texttt{yaobozhang@zju.edu.cn}}
\date{}

\begin{document}
\maketitle

\begin{abstract}
We organize relative-position mechanisms in attention as a learnable
Fourier--Jet--Affine position space.  The starting point is lag-shift
dynamics: a relative-position kernel is a response function of the lag
\(d=i-j\), and the one-step shift \((Ef)(d)=f(d+1)\) gives a compact
classification of finite structured responses through constant-coefficient
difference modules.  In this
view, RoPE supplies simple Fourier roots, Jordan-RoPE thickens these roots into
finite Fourier jets, and ALiBi supplies the repeated unit-root affine
direction.  NTK-aware RoPE scaling fits the same structure as a spectral flow
of simple Fourier roots: moving the frequency grid generates first
Fourier-jet tangent directions, while higher Taylor directions generate higher
jets.  \pjrope{} makes these jet directions explicit and learnable, and uses
the resulting space to measure task-level sector selection.

The framework separates scalar PJ-bias kernels from exact PJ-rotary feature
transforms, introduces sector-gate, effective-mass, functional-energy, and
leave-one-order-out diagnostics, and stabilizes high-order coordinates with
LC/rapidity compactification.  Controlled probes recover designed sectors;
synthetic teachers show trainable use; small byte-level language runs favor
NTK-aware RoPE plus affine recency; symbolic music-token streams keep LC/affine
variants strong with measurable high-order corrections; and LC diagnostics
quantify the stability--resolution tradeoff.
\end{abstract}

\section{Introduction}

Relative position representations have become a central part of long-context
attention in Transformers~\cite{Vaswani2017}.  RoPE encodes relative lag
through rotations generated by Fourier characters~\cite{Su2021}.  ALiBi adds a
monotone affine recency term directly to attention logits~\cite{Press2022}.
Relative attention and recurrence-based long-context mechanisms provide earlier
evidence that lag structure matters inside attention~\cite{Shaw2018,Dai2019}.
RoPE scaling methods alter how rotary phase is allocated outside the training
window~\cite{Chen2023,Peng2024,Ding2024}.
Jordan-RoPE takes a different step: it replaces a semisimple rotary frequency
with a finite Jordan block, producing derivative-like finite-jet features around
the same frequency~\cite{JordanRoPE2026}.

These methods are often described as competing drop-in designs.  That view
hides a useful common structure.  RoPE is a reduced Fourier character point.
Jordan-RoPE is a non-reduced thickening of that point.  High-order Jordan-RoPE
increases the jet multiplicity.  ALiBi contributes an affine,
translation-like recency direction.  This suggests a position-space
interpretation: useful attention may require both a frequency-jet sector and an
affine sector, with the task selecting how much of each to use.

An operational way to express this common structure is through the lag-shift
operator.  A relative-position kernel is a response function \(K(d)\) on the
lag axis.  Moving a key one token farther from the query acts by
\((EK)(d)=K(d+1)\).  If a positional mechanism has finite structured lag
dynamics, then the shifted responses \(K,EK,E^2K,\ldots\) span a
finite-dimensional space, equivalently \(P(E)K=0\) for some polynomial \(P\).
The roots of this difference equation classify the primitive responses: simple
unit-modulus roots give RoPE-like Fourier characters, roots inside the unit
disk give damped long-distance correlations, repeated nonzero roots give
Fourier jets, and the repeated unit root gives affine recency.  NTK-aware RoPE
scaling then has a natural interpretation as a spectral flow of simple Fourier
roots; its tangent directions are first Fourier jets, and higher Taylor
directions generate higher jets.

We use the term Poincar\'e-type only for the affine-completion pattern suggested
by this decomposition.  In physics, the Poincar\'e group is the affine extension
of the homogeneous Lorentz group obtained by adjoining
translations~\cite{Wigner1939,Hall2015},
\[
  \mathrm{ISO}(1,n-1)=\mathbb{R}^{1,n-1}\rtimes SO^+(1,n-1).
\]
Here the phrase has a narrower meaning: \pjrope{} forms an affine completion of
a homogeneous Fourier--jet positional representation.  RoPE and Jordan-RoPE
live in the homogeneous phase/jet sector, while ALiBi-like recency supplies the
additive affine direction.  This is a structural analogy, not a literal
spacetime symmetry of token sequences.

\pjrope{} turns this interpretation into a concrete relative-position
framework.  It defines a Fourier--Jet--Affine space of primitives and makes the
sector weights learnable.  The paper studies relative attention as the object of
interest.  Its scalar implementation, PJ-bias, is
an additive attention-logit kernel containing Fourier characters, damped finite
jets, affine recency terms, and LC compactified variants.  Its exact
feature-transform implementation, PJ-rotary, applies a relative action to query
and key features and is used to verify RoPE/Jordan-RoPE closure.  Keeping these
two regimes separate is essential: scalar PJ-bias can recover scalar kernels,
but it is not the same object as a rotary feature transform.

The framework also exposes a stability problem.  High-order jets carry powers
of the relative distance.  At long context, these powers can cause transformed
features, logits, and cache scales to grow dramatically.  Light-cone PJ replaces
raw distance by a compactified phase \(L\,\operatorname{asinh}(d/L)\) and
saturating amplitude \(d / \sqrt{d^2 + L^2}\).  This controls growth, but it
also compresses far-range phase resolution.  We treat that compression as an
explicit stability--resolution tradeoff.

PJ-RoPE is used here as a position-space diagnostic.  The experiments ask which
region of the Fourier--Jet--Affine space is selected by a task: controlled
kernels recover their designed sectors, synthetic sequence tasks test whether
these sectors can be used inside trainable attention, language runs concentrate
on affine/recency behavior, music-token streams show LC/affine behavior with
measurable high-order corrections, and LC diagnostics expose the
stability--resolution tradeoff.

\paragraph{Contributions.}
We make four contributions.  \textbf{Position-space formulation.}  We
formulate \pjrope{} as a Fourier--Jet--Affine relative-position space and
separate scalar PJ-bias kernels from exact PJ-rotary feature transforms.
\textbf{Sector containment and spectral-flow interpretation.}  We identify
RoPE, Jordan/high-order finite jets, ALiBi-like recency, and LC compactified
coordinates as sectors of the same lag-shift position space, and relate
NTK-aware RoPE scaling to Fourier-jet tangent directions.
\textbf{Adaptive diagnostics.}  We introduce sector gates, effective
mass, functional energy, and leave-one-order-out ablations to measure task-level
sector selection.  \textbf{Evidence chain.}  We evaluate sector recovery,
trainable use, natural-task allocation, and LC stability--resolution tradeoffs
across controlled kernels, synthetic sequence tasks, byte-level language, and
symbolic music-token streams.

\section{Background}

Let \(d=i-j\ge 0\) denote the relative lag from a query at position \(i\) to a
key at position \(j\).  RoPE represents this lag by a rotation.  In complex
notation, the primitive relative character is
\[
  \chi_\omega(d) = \exp(i\omega d).
\]
The real implementation exposes the corresponding cosine and sine components
\cite{Su2021}.  Long-context RoPE variants such as position interpolation,
YaRN, and LongRoPE change the phase schedule used beyond the training length
\cite{Chen2023,Peng2024,Ding2024}.

ALiBi uses a different primitive.  It adds a head-specific linear recency term
to the attention score~\cite{Press2022}.  In a relative-position basis, this is
an affine direction, not a Fourier phase.  This distinction matters because a
model may need local recency and oscillatory phase information for different
reasons.

Jordan-RoPE replaces a pure rotary block by a non-semisimple Jordan
block~\cite{JordanRoPE2026}.  A first-order Jordan correction contributes terms
of the form \(d\exp(i\omega d)\).  Higher-order blocks contribute higher powers
of \(d\) multiplied by the same Fourier character.  In differential language,
these are finite jets of the Fourier character curve.  This makes Jordan-RoPE a
local thickening of RoPE in frequency space.

\section{Related Work}

\paragraph{Transformers and relative position.}
The original Transformer adds absolute sinusoidal position signals to an
otherwise permutation-invariant attention layer~\cite{Vaswani2017}.  Relative
position representations instead inject pairwise lag structure directly into
self-attention~\cite{Shaw2018}, while Transformer-XL combines relative
position terms with segment-level recurrence for longer-context language
modeling~\cite{Dai2019}.  These works motivate our relative-attention
formulation of \pjrope{}.

\paragraph{Rotary and affine relative position.}
RoPE represents relative lag through a homogeneous rotary phase action, making
translation of positions appear as phase differences in query/key inner
products~\cite{Su2021}.  ALiBi takes the complementary scalar route: it adds a
head-specific linear recency bias directly to attention logits~\cite{Press2022}.
The two mechanisms are often treated as competing recipes, but for this paper
they supply different sectors of a relative-position space: a homogeneous phase
sector and an affine recency sector.

\paragraph{Long-context phase scaling and kernelized biases.}
Position interpolation, YaRN, LongRoPE, and XPos modify how rotary or relative
phase is allocated outside the training window~\cite{Chen2023,Peng2024,Ding2024,Sun2023}.
NTK-aware and adjusted-base-frequency RoPE variants can be viewed in the same
family of phase-schedule modifications: they change the effective RoPE
frequency grid used at long context
\cite{Bloc97NTK2023,Xiong2023LongContextScaling,Liu2024ScalingRoPE}.  In this
paper we use them not as a separate primitive sector, but as spectral-flow
baselines whose tangent directions connect naturally to Fourier jets.
Kernelized and functional relative-position methods such as KERPLE, FIRE, and
MEP
generalize scalar relative biases through kernel or learned-function
families~\cite{Chi2022,Li2024,Gao2024MEP}.  Hyperbolic bias methods such as
HyPE also use nonlinear distance coordinates for relative-position bias
\cite{Angelotti2023HyPE}.  These methods are strong baselines for
long-context language modeling and remain important reference lines.  \pjrope{}
asks which primitive a task selects when Fourier, finite-jet, affine, and
LC-stabilized coordinates are available in a shared space.

\paragraph{Relative-position primitive families.}
Existing relative-position mechanisms can also be grouped by the primitive
function of the lag that they introduce into attention.  Table-based and
bucketed methods learn local or discretized offsets.  Bias-centric methods add
scalar recency functions to the logits, including linear, bucketed,
logarithmic, or kernelized forms.  Rotary methods use Fourier phase characters
and modify their frequency schedules for long-context extrapolation.
Jordan-type methods add polynomially modulated Fourier terms.  \pjrope{}
follows this primitive-family view: the Fourier sector contains order-zero
phase characters, the finite-jet sector contains repeated-root Fourier jets,
the affine sector contains ALiBi-like recency, and the LC branch compactifies
high-order coordinates for long-context stability.

\paragraph{NoPE and length-generalization analysis.}
No-position-encoding studies show that causal Transformers can still acquire
positional information from the causal mask and training dynamics, and broader
comparisons find that the best length-generalization behavior depends strongly
on task and position mechanism~\cite{Haviv2022,Kazemnejad2023}.  This supports
the sector-selection framing: the empirical question is which region of the
position space is occupied under a given training regime and domain.

\paragraph{Algebraic and group-theoretic positional encodings.}
Several recent works study positional encoding through algebraic or group-action
lenses.  Algebraic positional encodings interpret positions as structured
operators~\cite{Kogkalidis2024}; LieRE generalizes RoPE through learned Lie
rotations~\cite{Ostmeier2024}; and GRAPE provides the closest existing
group-action unification framework for RoPE-like multiplicative rotations and
ALiBi-like additive biases~\cite{Zhang2026GRAPE}.  Within that landscape,
\pjrope{} emphasizes the non-semisimple Fourier--jet sector, where a rotary
frequency is replaced by a defective complex block, together with adaptive
sector diagnostics and LC/rapidity stabilization.

\paragraph{From Jordan-RoPE to PJ-RoPE.}
Jordan-RoPE extends RoPE by replacing a semisimple rotary frequency with a
finite Jordan block, yielding non-semisimple finite-jet corrections around a
Fourier point~\cite{JordanRoPE2026}.  \pjrope{} keeps that sector but broadens
the object of study: it adds affine recency as an explicit completion, exposes
learnable gates over sectors and orders, and introduces LC/rapidity coordinates
to stabilize high-order behavior at long range.  The upgrade is therefore from
a single non-semisimple rotary sector to a learnable Fourier--Jet--Affine
relative-position space.

\paragraph{Music sequence modeling.}
Music Transformer showed that self-attention and relative timing are natural
tools for symbolic music with long-range repetition~\cite{Huang2019}.  MAESTRO
provides aligned piano performance data with MIDI and audio~\cite{Hawthorne2019},
while MusicNet provides classical music recordings with note annotations
designed for transcription research~\cite{Thickstun2017}.  Our use of MAESTRO
and MusicNet is restricted to symbolic MIDI-derived token streams; it is not an
audio transcription benchmark.

\section{PJ-RoPE Position Space}

\pjrope{} studies a finite family of relative-position primitives indexed by
lag.  At the scalar-kernel level, these primitives are functions \(K(d)\) that
can be added to attention logits.  At the feature-transform level, they arise
from a one-parameter relative action \(G(d)\) applied to query and key features.
The paper uses both views, but keeps their claims separate.

The Fourier sector contains characters such as \(\cos(\omega d)\) and
\(\sin(\omega d)\).  The finite-jet sector thickens a Fourier point with terms
such as
\[
  (d/L)^r \exp(-cd/L)\cos(\omega d),
\]
and the corresponding sine components.  The affine sector contains constants
and linear recency terms such as \(-sd/L\).  The LC branch is a compactified
coordinate chart for high-order Fourier--jet behavior.

One algebraic way to view the same object is through constant-coefficient
difference modules.  Finite solutions of equations of the form \(P(E)K=0\) are
spanned by functions resembling
\[
  \binom{d}{r} z^d.
\]
The semisimple case \(z=\exp(i\omega), r=0\) gives the Fourier/RoPE sector.
The non-semisimple cases with \(r>0\) give finite jets.  The point \(z=1\) with
linear terms gives the affine/recency sector.  \pjrope{} packages these regions
into a learnable relative-position space with task-selected sector weights.
Appendix~\ref{app:difference-modules} develops this viewpoint formally:
simple roots give RoPE-like Fourier characters, repeated nonzero roots give
Jordan/Fourier jets, and the repeated unit root gives the ALiBi-like affine
direction.

\subsection{Poincar\'e-type affine completion}

The Poincar\'e-type terminology refers to an affine-completion pattern.  Let
\[
  \mathcal H_{\mathrm{FJ}}
  =
  \operatorname{span}
  \left\{
  (d/L)^r e^{-cd/L}e^{i\omega d}
  \right\}_{\omega,c,r}
\]
denote the homogeneous Fourier--jet sector.  This sector contains RoPE at
\(r=0\) and Jordan-RoPE at \(r>0\).  The affine recency sector is
\[
  \mathcal A_{\mathrm{rec}}
  =
  \operatorname{span}\{1,-d/L\}.
\]
\pjrope{} forms the finite relative-position module
\[
  \mathcal P_{\mathrm{PJ}}
  =
  \mathcal H_{\mathrm{FJ}}
  \oplus
  \mathcal A_{\mathrm{rec}}
  \oplus
  \mathcal H_{\mathrm{LC}},
\]
where \(\mathcal H_{\mathrm{LC}}\) is a light-cone compactified coordinate
chart for the high-order sector.

This is analogous to passing from a homogeneous group to an affine group.  The
homogeneous part supplies phase and finite-jet transformations; the affine part
supplies translation-like additive recency.  In this sense, RoPE/Jordan-RoPE
correspond to the homogeneous phase representation, while \pjrope{} is its
Poincar\'e-type affine extension.  At the scalar-kernel level this affine
completion is implemented as a finite direct-sum module with learned gates, not
as a literal semidirect-product group action.

\begin{definition}[PJ-bias kernel]
Fix a head \(h\), training scale \(L\), lag \(d\ge 0\), frequencies
\(\omega_{\ell h}\), damping rates \(c_{\ell h}\ge 0\), and maximum jet order
\(R\).  A scalar PJ-bias kernel is a learned additive attention term
\[
  K_h(d)=
  g_{\mathrm{FJ},h}K_{\mathrm{FJ},h}(d)
  +g_{\mathrm{aff},h}K_{\mathrm{aff},h}(d)
  +g_{\mathrm{LC},h}K_{\mathrm{LC},h}(d),
\]
where
\[
  K_{\mathrm{FJ},h}(d)=
  \sum_{\ell,r}
  \left(\frac{d}{L}\right)^r e^{-c_{\ell h}d/L}
  \left[
    a_{\ell r h}^{c}\cos(\omega_{\ell h}d)
    +a_{\ell r h}^{s}\sin(\omega_{\ell h}d)
  \right],
\]
\[
  K_{\mathrm{aff},h}(d)=b_{0,h}-s_h d/L.
\]
The sector weights are normalized gates
\[
  (g_{\mathrm{FJ},h},g_{\mathrm{aff},h},g_{\mathrm{LC},h})
  =
  \operatorname{softmax}(\eta_h),
\]
so the three branches form a learned allocation over Fourier--jet, affine, and
LC coordinates.
The LC branch is the compactified chart obtained by replacing raw phase and
amplitude in the same finite family by
\(\phi_L(d)=L\,\operatorname{asinh}(d/L)\) and
\(\beta_L(d)=d/\sqrt{d^2+L^2}\), for example
\[
  \beta_L(d)^r e^{-c_{\ell h}\phi_L(d)/L}
  \cos(\omega_{\ell h}\phi_L(d)).
\]
\end{definition}

\begin{proposition}[Sector containment]
Under parameter restrictions, the PJ position space contains the standard
relative-position primitives used in this paper.  The \(r=0\) Fourier sector is
the scalar character sector associated with RoPE.  Terms with \(r>0\) at fixed
\(\omega\) are finite-jet corrections corresponding to Jordan-RoPE and
high-order Jordan-RoPE.  The \(K_{\mathrm{aff}}\) branch contains ALiBi-like
affine recency.  The LC branch contains the compactified high-order variants
used for stabilization.  At the feature-transform level, exact PJ-rotary
recovers RoPE and Jordan-RoPE by using the corresponding semisimple or
finite-Jordan generator.
\end{proposition}

\begin{proof}[Proof sketch]
Setting \(g_{\mathrm{aff}}=g_{\mathrm{LC}}=0\) and \(r=0\) leaves ordinary
Fourier characters.  Keeping a fixed frequency and allowing \(r>0\) gives the
finite derivatives of that character curve, equivalently the polynomial
coordinates generated by a finite Jordan block.  Setting
\(g_{\mathrm{FJ}}=g_{\mathrm{LC}}=0\) gives the affine bias.  The LC sector is a
coordinate substitution applied to the same finite basis, so it is a
compactified subfamily; it does not assert a separate exact rotary action.
\end{proof}

\begin{figure}[!t]
\centering
\begin{tikzpicture}[
  >=Stealth,
  font=\small,
  axis/.style={->,line width=1.00pt,black!88},
  grid/.style={gray!36,dashed,line width=.32pt},
  faint/.style={gray!22,dashed,line width=.26pt},
  panelborder/.style={line width=.90pt,draw opacity=.88},
  ticklabel/.style={font=\scriptsize,text=black!72},
  fjtext/.style={font=\scriptsize,text=purple!66!black,inner sep=1pt},
  afftext/.style={font=\scriptsize,text=green!40!black,inner sep=1pt},
  lctext/.style={font=\scriptsize,text=teal!62!black,inner sep=1pt}
]

\node[font=\Large\bfseries] at (8.15,6.10)
  {Adaptive PJ-RoPE as a 3D spectral--jet sector space};
\node[font=\small\itshape,text=gray!66!black] at (8.15,5.70)
  {spectral base $\times$ jet fiber $\times$ sector slices, with head gates in $\Delta^2$};

\begin{scope}[
  x={(1.18cm,0cm)},
  y={(-0.90cm,0.56cm)},
  z={(0cm,1.10cm)}
]

\def\W{3.25}
\def\H{3.55}
\def\xFJ{0.80}
\def\xAFF{5.00}
\def\xLC{9.40}
\def\xMAX{14.30}

\foreach \s in {0,1,2}{\draw[grid] (.25,\s,0) -- (13.55,\s,0);}
\foreach \x in {.70,2.43,4.05,5.00,6.63,8.25,9.40,11.03,12.65,14.00}{
  \draw[faint] (\x,0,0) -- (\x,2.52,0);
}


\path[panelborder,draw=teal!68!black,fill=teal!45,fill opacity=.13]
  (\xLC,2,0) -- ++(\W,0,0) -- ++(0,0,\H) -- ++(-\W,0,0) -- cycle;
\node[font=\bfseries,text=teal!68!black] at ($ (\xLC,2,\H) + (.5*\W,0,-.34) $)
  {LC sector};

\def\xLCFIB{.36}
\def\xLCFORM{.70}
\draw[teal!76!black,line width=1.25pt] ($ (\xLC,2,0) + (\xLCFIB*\W,0,.58) $)
  -- ($ (\xLC,2,0) + (\xLCFIB*\W,0,2.16) $);
\foreach \r in {.58,.98,1.38,1.78,2.16}{
  \fill[teal!76!black] ($ (\xLC,2,0) + (\xLCFIB*\W,0,\r) $) circle (2.15pt);
}
\node[lctext] at ($ (\xLC,2,0) + (\xLCFIB*\W,0,2.47) $) {$J^{LC}_{\omega}$};

\draw[teal!55!black,opacity=.55,line width=.42pt]
  ($ (\xLC,2,0) + (.44*\W,0,1.38) $) -- ($ (\xLC,2,0) + (.56*\W,0,1.38) $);
\node[lctext,font=\scriptsize] at ($ (\xLC,2,0) + (.50*\W,0,1.48) $) {$\leftrightarrow$};

\node[lctext,align=center,font=\tiny,text width=1.45cm]
  at ($ (\xLC,2,0) + (.73*\W,0,1.53) $)
  {$\mathrm{mode}_r(d)$\\[-1pt]
   $\beta_L(d)^r$\\[-1pt]
   $e^{i\omega\phi_L(d)}$\\[-1pt]
   \textit{compact chart}};
\node[ticklabel,text=teal!65!black] at ($ (\xLC,2,0) + (.50*\W,0,.15) $) {$z_{\omega}$};

\path[panelborder,draw=green!52!black,fill=green!45,fill opacity=.13]
  (\xAFF,1,0) -- ++(\W,0,0) -- ++(0,0,\H) -- ++(-\W,0,0) -- cycle;
\node[font=\bfseries,text=green!43!black,align=center] at ($ (\xAFF,1,\H) + (.5*\W,0,-.43) $)
  {Affine / ALiBi\\sector};

\def\xAFIB{.43}
\def\yAONE{.72}
\def\yAD{1.78}
\draw[green!58!black,line width=1.28pt] ($ (\xAFF,1,0) + (\xAFIB*\W,0,\yAONE) $)
  -- ($ (\xAFF,1,0) + (\xAFIB*\W,0,\yAD) $);
\foreach \r in {\yAONE,\yAD}{
  \fill[green!58!black] ($ (\xAFF,1,0) + (\xAFIB*\W,0,\r) $) circle (2.25pt);
}
\node[afftext,anchor=east] at ($ (\xAFF,1,0) + (.30*\W,0,\yAONE) $) {$1$};
\node[afftext,anchor=east] at ($ (\xAFF,1,0) + (.30*\W,0,\yAD) $) {$d$};
\node[afftext] at ($ (\xAFF,1,0) + (\xAFIB*\W,0,2.18) $) {$J^1_1$};
\node[afftext,align=center,font=\tiny,text width=1.28cm]
  at ($ (\xAFF,1,0) + (.69*\W,0,1.34) $)
  {$A_{\rm rec}=$\\[-1pt]$\langle 1,-d/L\rangle$};
\node[ticklabel,text=green!43!black] at ($ (\xAFF,1,0) + (.50*\W,0,.15) $) {$z=1$};

\path[panelborder,draw=purple!70!black,fill=purple!45,fill opacity=.13]
  (\xFJ,0,0) -- ++(\W,0,0) -- ++(0,0,\H) -- ++(-\W,0,0) -- cycle;
\node[font=\bfseries,text=purple!70!black] at ($ (\xFJ,0,\H) + (.5*\W,0,-.34) $)
  {Fourier--jet sector};

\draw[purple!78!black,line width=1.30pt] ($ (\xFJ,0,0) + (.55*\W,0,.55) $)
  -- ($ (\xFJ,0,0) + (.55*\W,0,2.41) $);
\foreach \r in {.42,1.37,2.32}{
  \fill[purple!78!black] ($ (\xFJ,0,0) + (.55*\W,0,\r) $) circle (2.45pt);
}
\node[fjtext,anchor=east] at ($ (\xFJ,0,0) + (.39*\W,0,.55) $) {$e^{i\omega d}$};
\node[fjtext,anchor=east] at ($ (\xFJ,0,0) + (.39*\W,0,1.48) $) {$d\,e^{i\omega d}$};
\node[fjtext,anchor=east] at ($ (\xFJ,0,0) + (.39*\W,0,2.41) $) {$d^2e^{i\omega d}$};
\node[fjtext] at ($ (\xFJ,0,0) + (.55*\W,0,2.68) $) {$J^R_{\omega}$};

\node[fjtext,anchor=west] at ($ (\xFJ,0,0) + (.70*\W,0,.55) $) {$r=0$};
\node[fjtext,anchor=west] at ($ (\xFJ,0,0) + (.70*\W,0,1.48) $) {$r=1$};
\node[fjtext,anchor=west] at ($ (\xFJ,0,0) + (.70*\W,0,2.41) $) {$r=2$};
\node[ticklabel,text=purple!70!black] at ($ (\xFJ,0,0) + (.13*\W,0,.15) $) {$z_{\omega_1}$};
\node[ticklabel,text=purple!70!black] at ($ (\xFJ,0,0) + (.55*\W,0,.15) $) {$z_{\omega}$};
\node[ticklabel,text=purple!70!black] at ($ (\xFJ,0,0) + (.87*\W,0,.15) $) {$z_{\omega_2}$};

\draw[axis] (0,0,0) -- (\xMAX,0,0);
\node[font=\small,text=black!82] at (9.0,0,-.78) {spectral base $z_{\omega}=e^{-c/L+i\omega}$};
\draw[axis] (0,0,0) -- (0,0,4.15);
\node[font=\small,text=black!82,anchor=east] at (-.12,0,3.97) {jet order $r$};
\draw[axis] (0,0,0) -- (0,2.80,0);
\node[font=\small,text=black!82,align=center,anchor=east] at (-.55,2.55,.15) {sector\\slices};

\node[font=\scriptsize,text=purple!70!black,anchor=east] at (-.70,.18,.13) {FJ};
\node[font=\scriptsize,text=green!43!black,anchor=east] at (-.70,1.00,.13) {affine};
\node[font=\scriptsize,text=teal!65!black,anchor=east] at (-.70,2.00,.13) {LC};

\end{scope}


\begin{scope}[shift={(15.25,2.55)},scale=1.05]
  \coordinate (FJg) at (0,0);
  \coordinate (AG)  at (1.60,0);
  \coordinate (LCg) at (.80,1.38);
  \draw[line width=.82pt,black!80] (FJg)--(AG)--(LCg)--cycle;
  \node[font=\scriptsize,text=purple!70!black,anchor=north east]
    at ($ (FJg) + (-.05,-.08) $) {$g_{\rm FJ}$};
  \node[font=\scriptsize,text=green!43!black,anchor=north west]
    at ($ (AG) + (.05,-.08) $) {$g_{\rm aff}$};
  \node[font=\scriptsize,text=teal!65!black,anchor=west] at (.88,1.43) {$g_{\rm LC}$};
  \node[font=\small\bfseries,text=black!78] at (.80,1.88) {sector gate $\Delta^2$};
  \fill[orange!30] (.46,.37) circle (2.1pt);
  \fill[orange!55] (.63,.50) circle (2.1pt);
  \fill[orange!90!red] (.84,.63) circle (3.0pt);
  \draw[->,orange!90!red,line width=.72pt] (.46,.37) to[bend right=14] (.84,.63);
  \node[font=\tiny,text=gray!65!black,align=center,text width=2.45cm] at (.80,-.93)
    {$K=g_{\rm FJ}K_{\rm FJ}$\\[-1pt]
     $+g_{\rm aff}K_{\rm aff}$\\[-1pt]
     $+g_{\rm LC}K_{\rm LC}$};
\end{scope}

\end{tikzpicture}
\caption{\pjrope{} position space.  The framework organizes homogeneous
Fourier--jet coordinates, affine recency, and LC-stabilized high-order
coordinates into a shared relative-position space with adaptive sector
diagnostics.  The LC branch is a compactified coordinate chart for high-order
Fourier--jet behavior, not an exact rotary group action.}
\label{fig:geometry}
\end{figure}
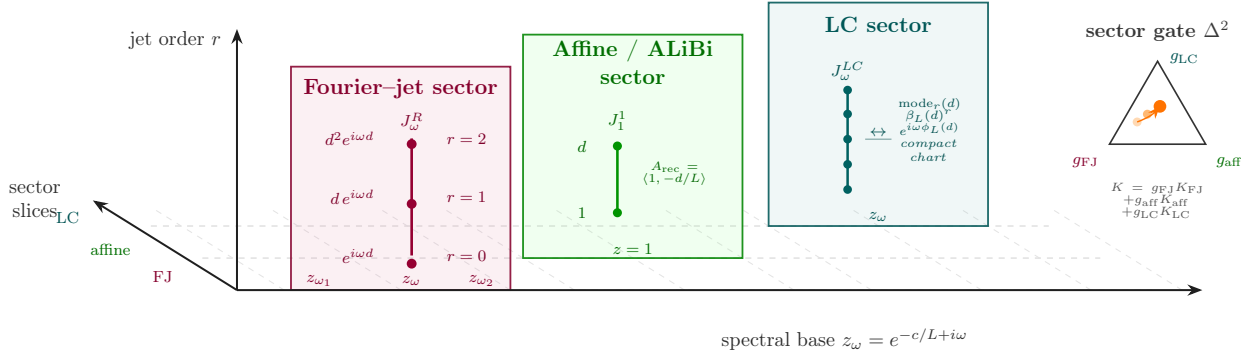
\FloatBarrier

\section{Implementation Regimes}

The theory permits several implementations.  The current experiments separate
four regimes because each supports a different claim.

\begin{table}[H]
\centering
\caption{Implementation regimes.  The table separates scalar additive kernels
from exact feature-transform claims.}
\label{tab:regimes}
\small
\begin{tabularx}{\linewidth}{lXcX}
\toprule
Regime & Object & Exact feature action? & Main use \\
\midrule
PJ-bias & Scalar additive attention kernel & no & Scalable kernel diagnostics and LM experiments \\
PJ-rotary exact & Q/K feature transform & yes & Closure tests and feature-transform control \\
Scaled-exact PJ-rotary & Normalized generator feature transform & yes & Controlled feature-transform baseline \\
LC-PJ bias & Compactified scalar attention kernel & no & Long-context stabilization \\
\bottomrule
\end{tabularx}
\end{table}

This taxonomy separates the claims supported by each implementation.  PJ-bias can recover scalar
Fourier, jet, affine, and LC kernels, but it is not itself the RoPE feature
transform.  Exact PJ-rotary is the object used for feature-level representation
claims.  LC-PJ bias is a stabilized scalar path: it is designed to control
long-context growth, not to provide an exact rotary group action.

\section{Adaptive PJ}

Adaptive PJ turns the position space into a measurable selection problem.  Each
head receives sector gates over the Fourier/jet, affine, and light-cone
branches.  Within the FJ or LC branch, an order spectrum allocates mass across
jet orders.  The diagnostics determine whether the learned position primitive
moves toward the sector implied by the teacher or task.

The three diagnostics answer complementary questions.  Effective mass asks
where gated parameter magnitude is allocated.  Functional energy asks which
orders contribute to the realized kernel over the evaluation window.
Leave-one-order-out asks whether removing an order changes the fitted function.

Raw parameters are not enough for this diagnostic.  Different basis functions
can have different scales and supports, so the paper reports both a
parameter-side quantity, \emph{effective mass}, and a function-side quantity,
\emph{functional energy}.  Let
\[
  g_h=\operatorname{softmax}(\eta_h)
  =(g_{h,\mathrm{FJ}},g_{h,\mathrm{aff}},g_{h,\mathrm{LC}})
\]
be the sector gates for head \(h\).  Within the FJ and LC sectors, let
\(\alpha^B_{h,r}\) be the conditional order spectrum for branch
\(B\in\{\mathrm{FJ},\mathrm{LC}\}\), and let
\[
  |\zeta^B_{h,r}| =
  \sqrt{(\zeta^{B,c}_{h,r})^2+(\zeta^{B,s}_{h,r})^2}
\]
denote the signed sine/cosine amplitude magnitude at order \(r\).  The
branch-specific effective mass is
\[
  M^B_{h,r}=g_{h,B}\alpha^B_{h,r}|\zeta^B_{h,r}|,
  \qquad B\in\{\mathrm{FJ},\mathrm{LC}\}.
\]
The order-level effective mass used in the diagnostic plots is
\[
  \bar M_{h,r}
  =
  \frac{M^{\mathrm{FJ}}_{h,r}+M^{\mathrm{LC}}_{h,r}}
       {\sum_{j=0}^{R}
        \left(M^{\mathrm{FJ}}_{h,j}+M^{\mathrm{LC}}_{h,j}\right)+\epsilon}.
\]
This is a parameter-side answer to the question: where did the model allocate
its gated high-order amplitude?  The affine branch has no jet order, so affine
selection is reported through \(g_{h,\mathrm{aff}}\) and the learned slope.

Functional energy measures realized-kernel scale over a lag window.  Write the
order-\(r\) realized component as
\[
  C_{h,r}(d)
  =
  g_{h,\mathrm{FJ}} C^{\mathrm{FJ}}_{h,r}(d)
  +
  g_{h,\mathrm{LC}} C^{\mathrm{LC}}_{h,r}(d),
\]
where each \(C^B_{h,r}\) includes the corresponding basis function, order
weight \(\alpha^B_{h,r}\), amplitude \(\zeta^B_{h,r}\), damping, and LC
coordinate substitution if applicable.  For a window \(W\) with weights
\(w_d\), define
\[
  \|f\|_{2,W}
  =
  \left(\sum_{d\in W} w_d f(d)^2\right)^{1/2}.
\]
The functional energy ratio is
\[
  E_{h,r}(W)
  =
  \frac{\|C_{h,r}\|_{2,W}}
       {\left\|\sum_{j=0}^{R} C_{h,j}\right\|_{2,W}+\epsilon}.
\]
Unlike \(\bar M_{h,r}\), this quantity accounts for finite-window scale,
damping, LC compactification, phase cancellation, and extrapolation length.  It
is therefore not forced to sum to one when order components are not orthogonal;
it is a diagnostic of realized functional scale.

Finally, leave-one-order-out asks whether the order is needed for the fitted
function:
\[
  \Delta^{\mathrm{LOO}}_{h,r}(W)
  =
  \operatorname{MSE}_W\!\left(K_h-C_{h,r},y\right)
  -
  \operatorname{MSE}_W\!\left(K_h,y\right).
\]
A useful order should have coherent evidence across these views: nontrivial
effective mass, nontrivial functional energy on the evaluation window, and a
positive leave-one-order-out delta.  This separates parameter selection from
actual function-level contribution.

\section{Light-cone PJ}

Raw high-order jets contain powers of distance.  At long context, those powers
can make transformed features, logits, and cache scales grow rapidly.  LC-PJ
replaces the raw coordinate with a compactified phase and a saturating
amplitude:
\[
  \phi_L(d)=L\,\operatorname{asinh}(d/L),\qquad
  \beta_L(d)=\frac{d}{\sqrt{d^2+L^2}}.
\]
The LC coordinate has a rapidity form.  Writing
\[
  d/L=\sinh\eta,
\]
we obtain
\[
  \phi_L(d)=L\eta,\qquad
  \beta_L(d)=\tanh\eta.
\]
Thus LC-PJ replaces raw distance by a rapidity-like phase coordinate and a
velocity-like bounded amplitude.  High-order jet powers are applied to
\(\beta_L(d)\), so they remain bounded at large lag:
\[
  |\beta_L(d)|\le 1.
\]
The price is that the rapidity coordinate grows only logarithmically in the far
field, and
\[
  \partial_d\phi_L(d)=\frac{1}{\sqrt{1+(d/L)^2}},
\]
so phase resolution is compressed at long range.  A wrong-scale controlled
contrast using \(\operatorname{asinh}(d/L)\) without the outer \(L\) is stable in
norm but collapses usable phase resolution.  The central LC claim is therefore
a measurable stability--resolution tradeoff.

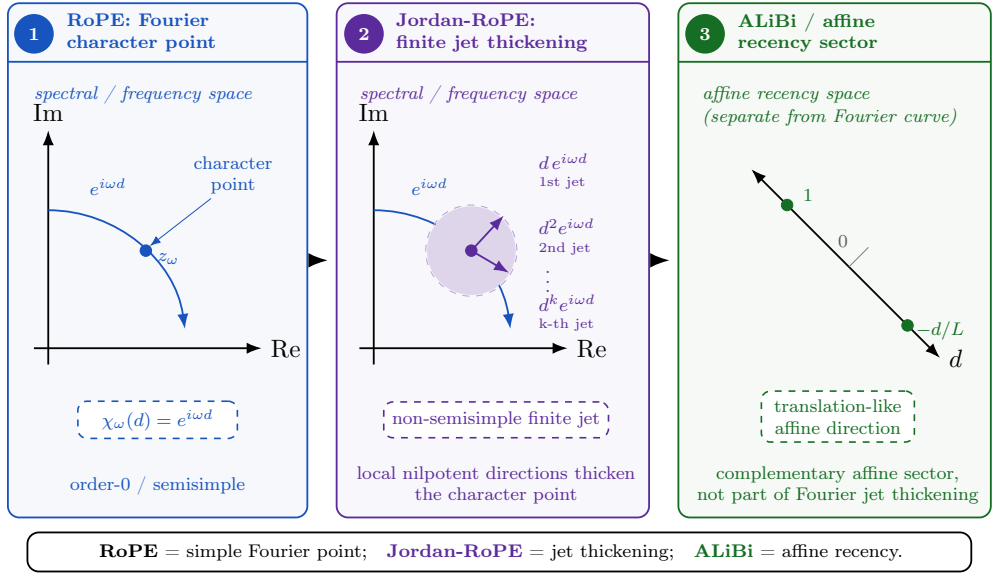
\begin{figure}[t]
\centering
\resizebox{\linewidth}{!}{\begin{tikzpicture}[
  x=1cm,y=1cm,
  >=Latex,
  panel/.style={rounded corners=5pt, line width=0.8pt, fill=#1!4, draw=#1},
  note/.style={rounded corners=4pt, dashed, line width=0.7pt, draw=#1, fill=white, text=#1},
  dot/.style={circle, inner sep=0pt, minimum size=4pt, fill=#1, draw=#1},
  every node/.style={font=\sffamily}
]
\node[font=\sffamily\bfseries\Large] at (7.8,9.55) {From Fourier Character Points to Jordan Jet Thickenings and Affine Recency};

\draw[panel=pjblue] (0.15,1.05) rectangle (4.75,8.95);
\draw[panel=pjpurple] (5.20,1.05) rectangle (10.00,8.95);
\draw[panel=pjgreen] (10.45,1.05) rectangle (15.25,8.95);
\draw[pjblue, line width=0.8pt] (0.15,8.05) -- (4.75,8.05);
\draw[pjpurple, line width=0.8pt] (5.20,8.05) -- (10.00,8.05);
\draw[pjgreen, line width=0.8pt] (10.45,8.05) -- (15.25,8.05);

\node[circle, fill=pjblue, text=white, font=\bfseries\scriptsize] at (0.55,8.48) {1};
\node[anchor=west, text=pjblue, font=\bfseries\scriptsize, align=left] at (0.92,8.50) {RoPE: Fourier\\character point};
\node[circle, fill=pjpurple, text=white, font=\bfseries\scriptsize] at (5.62,8.48) {2};
\node[anchor=west, text=pjpurple, font=\bfseries\scriptsize, align=left] at (5.98,8.50) {Jordan-RoPE:\\finite jet thickening};
\node[circle, fill=pjgreen, text=white, font=\bfseries\scriptsize] at (10.86,8.48) {3};
\node[anchor=west, text=pjgreen, font=\bfseries\scriptsize, align=left] at (11.23,8.50) {ALiBi / affine\\recency sector};

\node[anchor=west, text=pjblue, font=\itshape\scriptsize] at (0.42,7.55) {spectral / frequency space};
\begin{scope}[shift={(0.55,3.65)}]
  \coordinate (O) at (0.22,0);
  \def\Runit{2.12}
  \coordinate (Zomega) at ($(O)+(45:\Runit)$);
  \draw[->, thick] (0,0) -- (3.50,0) node[right] {\(\mathrm{Re}\)};
  \draw[->, thick] (0.22,-0.25) -- (0.22,3.35) node[above] {\(\mathrm{Im}\)};
  \draw[pjblue, thick, ->] ($(O)+(90:\Runit)$)
    arc[start angle=90,end angle=8,radius=\Runit];
  \node[dot=pjblue, minimum size=5pt] at (Zomega) {};
  \draw[pjblue, ->] (2.72,2.30) -- ($(Zomega)+(.06,.07)$);
  \node[text=pjblue, font=\scriptsize] at (1.14,2.48) {\(e^{i\omega d}\)};
  \node[text=pjblue, font=\scriptsize, align=center] at (3.06,2.66) {character\\point};
  \node[text=pjblue, font=\scriptsize] at ($(Zomega)+(.35,-.12)$) {\(z_\omega\)};
\end{scope}
\node[note=pjblue, font=\scriptsize, minimum width=2.45cm, minimum height=0.55cm] at (2.45,2.55) {\(\chi_\omega(d)=e^{i\omega d}\)};
\node[text=pjblue, font=\scriptsize] at (2.45,1.55) {order-0 / semisimple};

\node[anchor=west, text=pjpurple, font=\itshape\scriptsize] at (5.42,7.55) {spectral / frequency space};
\begin{scope}[shift={(5.55,3.65)}]
  \coordinate (O) at (0.22,0);
  \def\Runit{2.12}
  \coordinate (Zomega) at ($(O)+(45:\Runit)$);
  \draw[->, thick] (0,0) -- (3.20,0) node[right] {\(\mathrm{Re}\)};
  \draw[->, thick] (0.22,-0.25) -- (0.22,3.35) node[above] {\(\mathrm{Im}\)};
  \draw[pjblue, thick, ->] ($(O)+(90:\Runit)$)
    arc[start angle=90,end angle=8,radius=\Runit];
  \foreach \r in {0.30,0.50,0.70} {
    \draw[pjpurple!45, dashed] (Zomega) circle (\r);
  }
  \fill[pjpurple!20] (Zomega) circle (0.68);
  \node[dot=pjpurple, minimum size=5pt] at (Zomega) {};
  \draw[pjpurple, thick, ->] (Zomega) -- ($(Zomega)+(.50,.54)$);
  \draw[pjpurple, thick, ->] (Zomega) -- ($(Zomega)+(.58,-.34)$);
  \node[text=pjblue, font=\scriptsize] at (1.08,2.48) {\(e^{i\omega d}\)};
  \node[text=pjpurple, anchor=west, align=left, font=\scriptsize]
    at (2.62,2.72) {\(d\,e^{i\omega d}\)\\[-2pt]{\tiny 1st jet}};
  \node[text=pjpurple, anchor=west, align=left, font=\scriptsize]
    at (2.62,1.72) {\(d^2e^{i\omega d}\)\\[-2pt]{\tiny 2nd jet}};
  \node[text=pjpurple, anchor=west, font=\scriptsize] at (2.72,1.13) {\(\vdots\)};
  \node[text=pjpurple, anchor=west, align=left, font=\scriptsize]
    at (2.62,0.55) {\(d^k e^{i\omega d}\)\\[-2pt]{\tiny k-th jet}};
\end{scope}
\node[note=pjpurple, font=\scriptsize, minimum width=3.0cm, minimum height=0.55cm] at (7.65,2.55) {non-semisimple finite jet};
\node[text=pjpurple, align=center, font=\scriptsize] at (7.65,1.55) {local nilpotent directions thicken\\the character point};

\node[anchor=west, text=pjgreen, font=\itshape\scriptsize] at (10.68,7.55) {affine recency space};
\node[anchor=west, text=pjgreen, font=\itshape\scriptsize] at (10.68,7.20) {(separate from Fourier curve)};
\begin{scope}[shift={(11.02,3.15)}]
  \draw[<->, thick] (0.55,3.25) -- (3.45,0.35) node[right] {\(d\)};
  \node[dot=pjgreen, minimum size=5pt] at (1.10,2.70) {};
  \node[text=pjgreen, font=\scriptsize] at (1.42,2.86) {\(1\)};
  \node[font=\scriptsize, text=pjgray] at (1.96,2.15) {\(0\)};
  \draw[gray] (2.05,1.75) -- (2.35,2.05);
  \node[dot=pjgreen, minimum size=5pt] at (2.95,0.85) {};
  \node[text=pjgreen, font=\scriptsize] at (3.45,0.78) {\(-d/L\)};
\end{scope}
\node[note=pjgreen, align=center, font=\scriptsize, minimum width=2.35cm, minimum height=0.75cm] at (12.90,2.60) {translation-like\\affine direction};
\node[text=pjgreen, align=center, font=\scriptsize] at (12.90,1.55) {complementary affine sector,\\not part of Fourier jet thickening};

\draw[->, line width=1.25pt] (4.84,5.00) -- (5.07,5.00);
\draw[->, line width=1.25pt] (10.09,5.00) -- (10.32,5.00);

\draw[rounded corners=5pt, line width=0.8pt, fill=white] (0.45,0.22) rectangle (14.95,0.82);
\node[font=\scriptsize, align=center, text width=13.2cm] at (7.7,0.52) {{\bfseries RoPE} = simple Fourier point;\quad {\color{pjpurple}\bfseries Jordan-RoPE} = jet thickening;\quad {\color{pjgreen}\bfseries ALiBi} = affine recency.};

\end{tikzpicture}}
\caption{Root-level difference-module schematic.  RoPE is a simple Fourier
root, Jordan-RoPE is a repeated nonzero root that generates finite Fourier
jets, and ALiBi supplies the repeated unit-root affine recency direction.}
\label{fig:root-schematic}
\end{figure}

\section{Summary of Questions and Observations}

Table~\ref{tab:claims} summarizes the experimental questions, observations,
and readings used in the rest of the paper.

\begin{table}[H]
\centering
\caption{Summary of questions, observations, and readings.}
\label{tab:claims}
\small
\begin{tabularx}{\linewidth}{p{0.24\linewidth}p{0.40\linewidth}p{0.26\linewidth}}
\toprule
Question & Observation & Reading \\
\midrule
Does PJ contain Fourier, jet, and affine primitives? &
Fixed-kernel probes recover the intended sectors. &
The scalar PJ space spans the designed primitive families. \\
Which sector do controlled teachers select? &
Adaptive gates, effective mass, functional energy, and leave-one-order-out diagnostics move toward the teacher sector. &
Sector selection is measurable beyond raw parameter size. \\
Can trainable attention use the sectors? &
Signed jet teachers, attention recency teachers, and LC-core teachers activate the corresponding branches. &
The sectors remain usable inside a small causal-attention model. \\
What do language and music-token runs select? &
Language runs favor NTK+affine at 32768 tokens, while music-token streams keep LC/affine variants strong with small high-order mass. &
Task allocation differs across domains and training regimes. \\
What does LC change? &
LC bounds Q/K and cache/logit scale while reducing far-range phase span. &
LC gives a stability--resolution tradeoff. \\
\bottomrule
\end{tabularx}
\end{table}

Detailed reproducibility notes are placed outside the main narrative.

\section{Experiments}

The experiments describe how different tasks occupy the PJ position space.
First, fixed kernels test scalar containment.  Second, adaptive diagnostics
turn controlled teachers into sector-allocation measurements.  Third, synthetic
sequence tasks test whether the sectors can be used inside trainable attention.
Fourth, byte-level language and music-token runs show how natural streams
allocate mass under the small-model training regimes used here.  Finally, LC
diagnostics measure the stability--resolution exchange introduced by the
compactified chart.

The experiments were run in separate diagnostic suites.  We present them by
their role in the argument, and a separate reproducibility bundle records the
exact source files.

\paragraph{Experimental setup.}
Unless otherwise noted, trainable sequence experiments use small causal
Transformers trained at short context and evaluated by validation
cross-entropy at longer contexts.  The default trainable model has two layers,
embedding dimension \(96\), four attention heads, and MLP ratio \(2\), giving
roughly \(0.20\)M parameters for byte-level language runs and \(0.15\)M
parameters for synthetic query classifiers.  The language boundary tests use
byte-level tokenization, train at context length \(1024\), and report
three-seed sampled 32768-byte-token stress evaluations on Tiny Shakespeare,
WikiText-2, and a Project Gutenberg War and Peace corpus.  A 32768-token
byte-level window is about 32KB of raw text; these rows are tokenizer-level
extrapolation stress tests, not document-scale semantic context evaluations.
The music-token tests use symbolic MIDI-derived byte
streams, not audio labels; the 32768-token MAESTRO and MusicNet summaries train
at context length \(512\) and report three model seeds with one long-context
evaluation batch per seed.  MAESTRO controls use a random 128-file train-split
MIDI token stream with controller/pedal events.  MusicNet selector A and
selector B are two random 64-piece reference-MIDI program-token streams,
sampled with seeds 29 and 37 from the MusicNet reference-MIDI archive.
Synthetic query-LM bridges freeze Q/K content scores in the main setting so
that the positional branch must carry the teacher signal.

\subsection{Fixed-kernel diagnostics}

The first diagnostic isolates expressivity at the scalar-kernel level.  We fit
fixed PJ bases to target kernels representing pure phase, first/second/third
jets, affine recency, mixed recency-plus-weak-jet behavior, and LC targets.  A
successful result includes both low loss and sector match: the selected basis
should match the intended sector.

The recovery summary shows that phase targets are recovered by RoPE-like bases,
jet targets by the corresponding finite-jet bases, affine targets by affine
bases, and LC targets by LC bases.  Table~\ref{tab:sector-recovery} gives the
numerical sector summary, and Figure~\ref{fig:recovery} gives the visual
summary.  The fixed-kernel results give the scalar containment picture: phase,
finite-jet, affine, and LC targets are recovered by their corresponding
sectors.

\subsection{Adaptive sector recovery}

The fixed-basis experiment asks whether the space is expressive.  The adaptive
experiment asks whether a learned PJ kernel can find the right region of that
space.  Controlled teachers are chosen so that the expected sector is known:
Fourier teachers should select order-zero FJ mass, higher jet teachers should
shift energy to higher FJ orders, affine teachers should activate the recency
branch, and LC teachers should select the light-cone branch.

Adaptive PJ turns the position space into an observable allocation problem.  We
report sector gates, effective mass, functional order energy, and
leave-one-order-out deltas side by side because damping, window length, LC
compactification, and phase cancellation can change the realized functional
contribution of a learned coefficient.  The main diagnostics are shown in
Figure~\ref{fig:recovery}.  Adaptive PJ moves toward the teacher sector when
the teacher is known.

\begin{table}[H]
\centering
\caption{Fixed-kernel and adaptive sector recovery summary.  Gates are reported
as FJ/A/LC.  The top order is the dominant functional-energy order inside the
selected high-order branch; \textemdash{} indicates that the selected sector
has no jet order.}
\label{tab:sector-recovery}
\small
\begin{tabularx}{\linewidth}{lXrcc}
\toprule
Target & Selected sector & Eval MSE & Gates & Top order \\
\midrule
Phase & Fourier / RoPE-like & \(7.62{\times}10^{-6}\) & 0.91/0.06/0.03 & 0 \\
First jet & first finite jet & \(3.70{\times}10^{-4}\) & 0.87/0.06/0.06 & 1 \\
Second jet & second finite jet & \(2.68{\times}10^{-3}\) & 0.89/0.05/0.05 & 2 \\
Third jet & third finite jet & \(7.63{\times}10^{-3}\) & 0.79/0.12/0.08 & 3 \\
Linear & affine recency & \(2.76{\times}10^{-6}\) & 0.08/0.85/0.07 & \textemdash \\
Recency + weak jet & affine with weak jet & \(1.40{\times}10^{-4}\) & 0.23/0.60/0.17 & 2 \\
LC core & LC compactified & \(6.96{\times}10^{-5}\) & 0.04/0.04/0.92 & 0 \\
LC + affine & LC plus affine & \(1.72{\times}10^{-3}\) & 0.04/0.12/0.83 & 0 \\
\bottomrule
\end{tabularx}
\end{table}

\begin{figure}[H]
\centering
\includegraphics[width=0.48\linewidth]{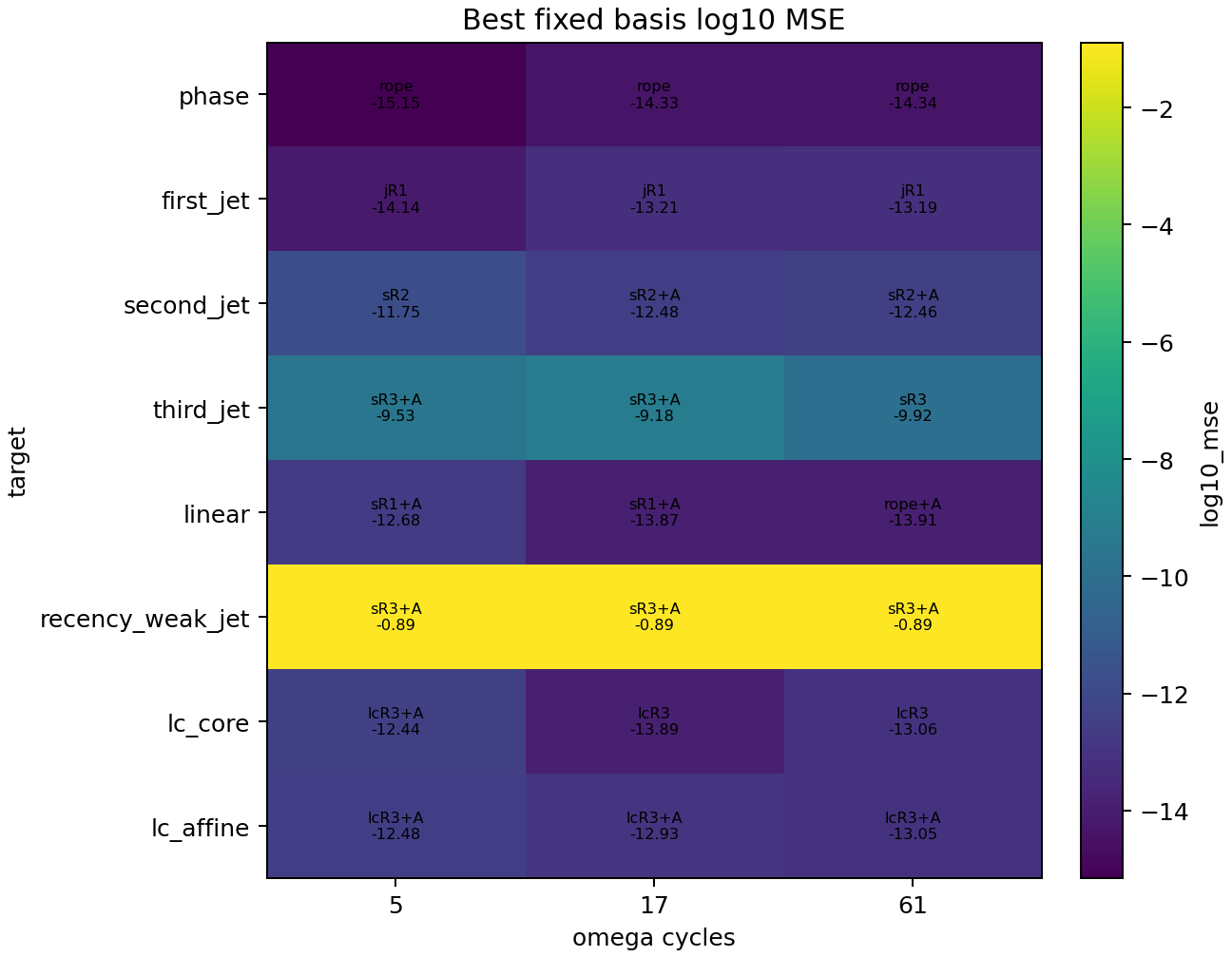}
\includegraphics[width=0.48\linewidth]{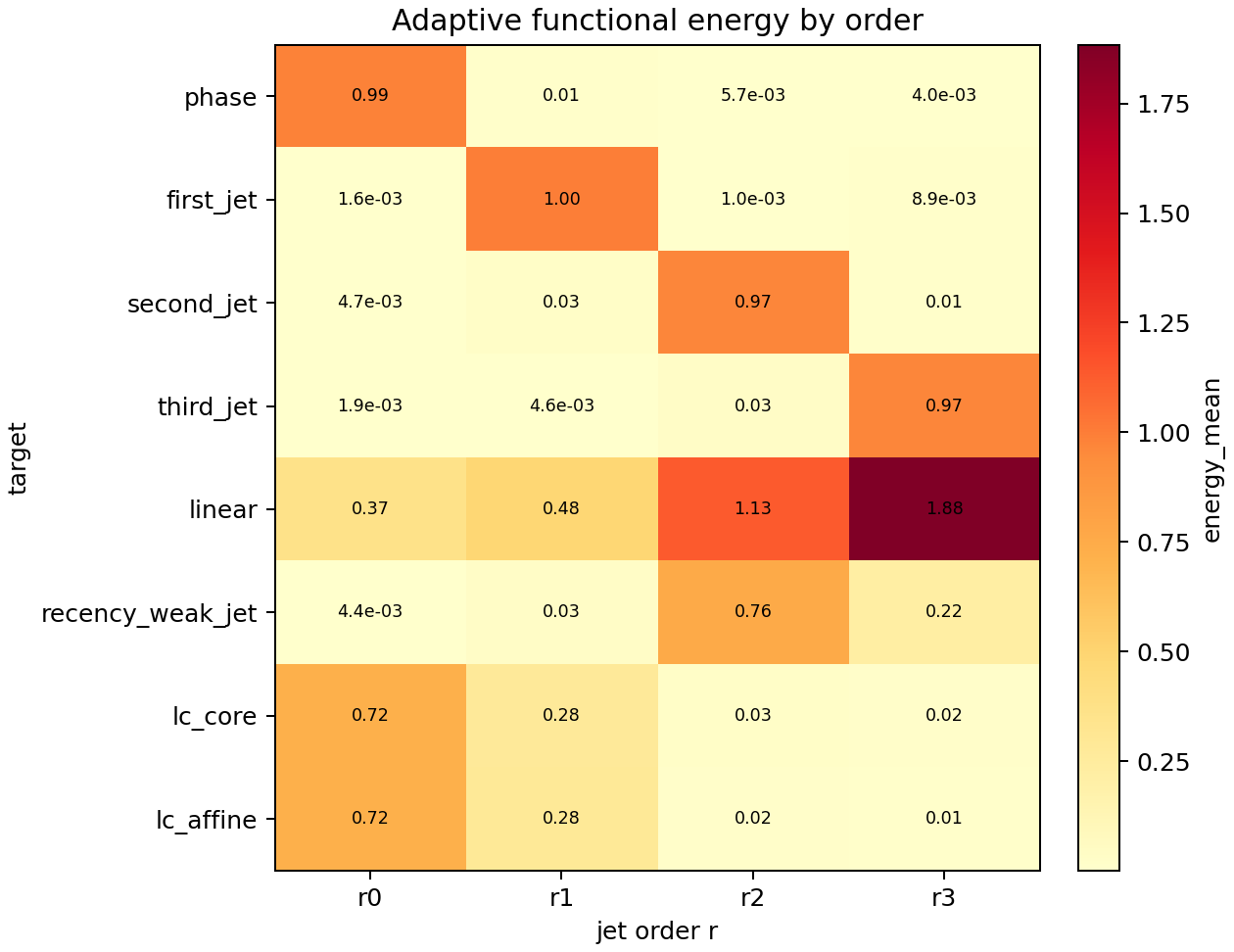}
\includegraphics[width=0.62\linewidth]{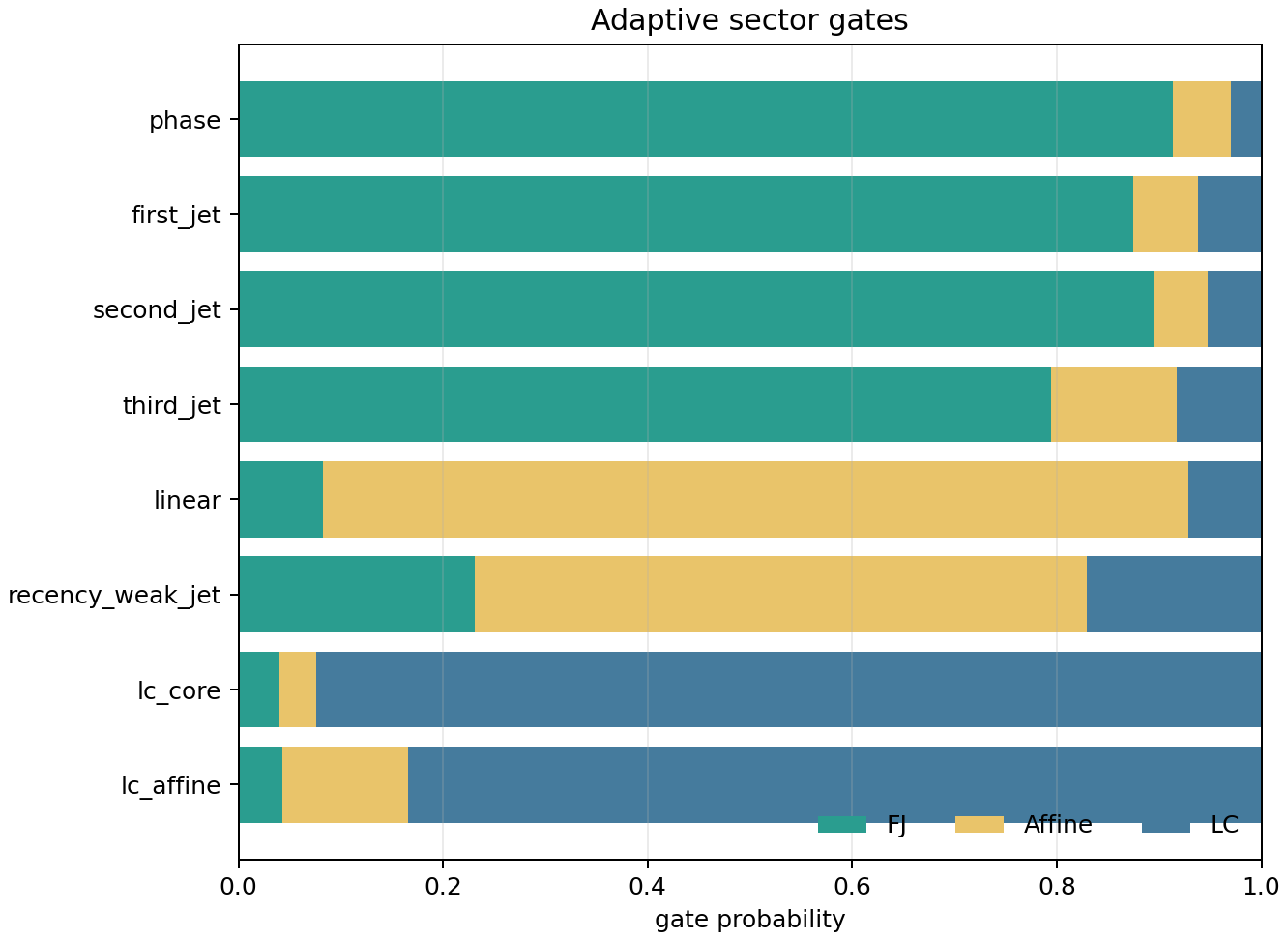}
\caption{Fixed-kernel and adaptive sector recovery.  Top left: fixed-basis
recovery error.  Top right: adaptive functional order energy.  Bottom: adaptive
sector gates.}
\label{fig:recovery}
\end{figure}

\subsection{Synthetic sequence bridge}

Static kernel recovery leaves open whether a Transformer can use the same
structure.  The synthetic sequence bridge therefore places PJ sectors inside a
small trainable causal-attention model.  The tasks are controlled query-LM
problems with signed jet teachers, affine attention teachers, and LC-core
teachers.  In the main bridge setting, Q/K content scores are frozen so that the
positional branch must carry the relevant signal.

The evidence shows that affine teachers are solved by the affine sector, signed
first/second-jet teachers are solved by FJ variants under multi-length training,
and LC-core teachers activate the LC path.  These synthetic tasks isolate the
trainable-attention step: the positional branch carries the teacher signal, and
the learned gates show whether the corresponding sector is used.

\begin{figure}[H]
\centering
\includegraphics[width=0.95\linewidth]{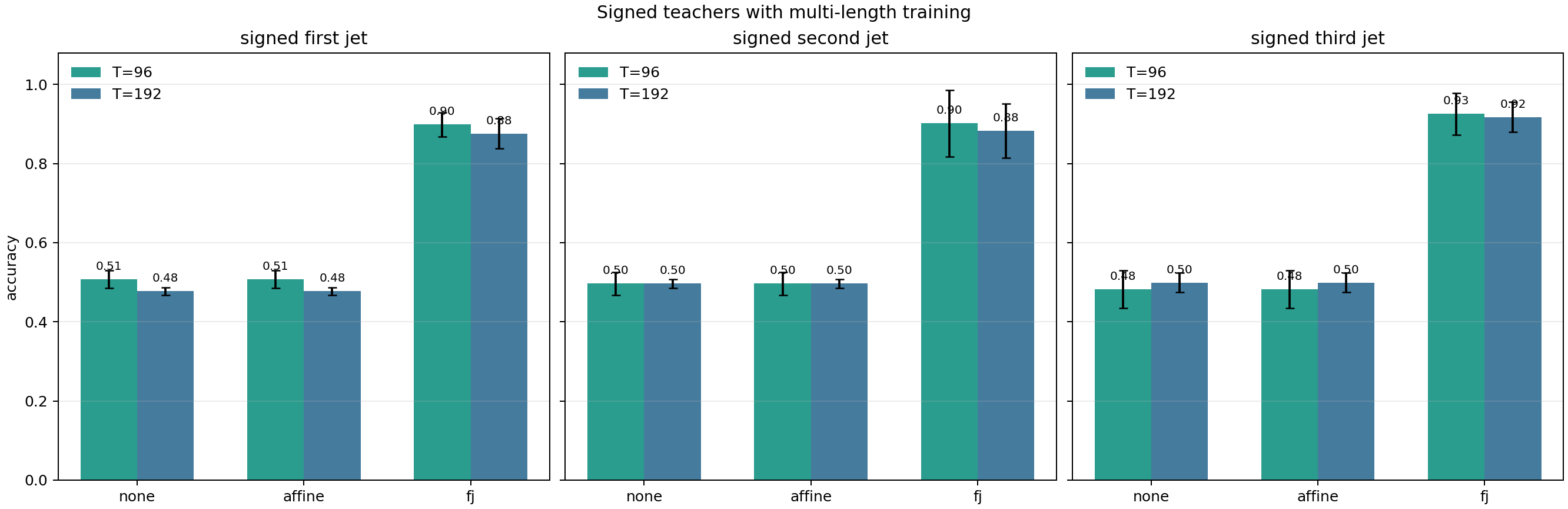}
\includegraphics[width=0.52\linewidth]{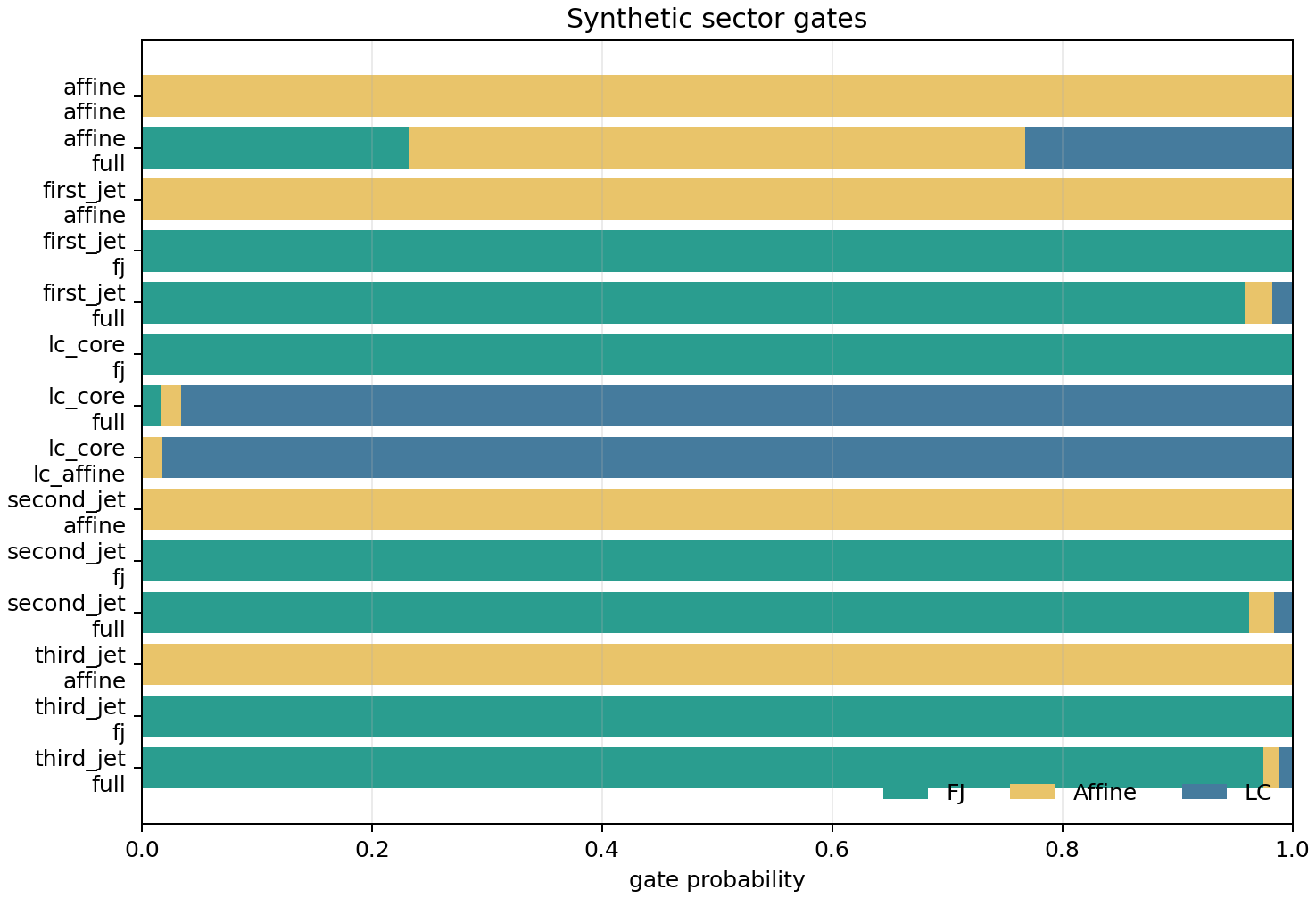}
\includegraphics[width=0.44\linewidth]{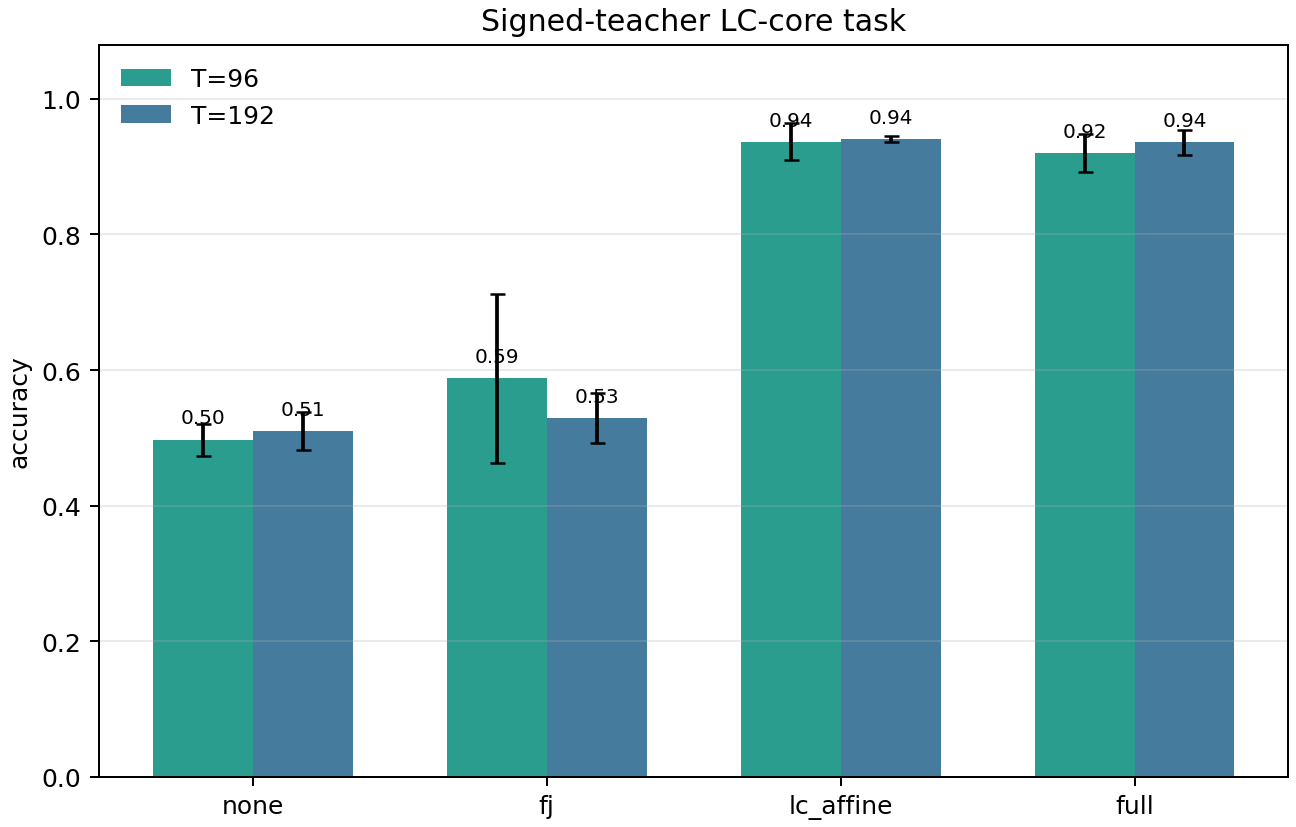}
\caption{Synthetic sequence bridge.  Top: signed jet teacher accuracy under
multi-length training.  Bottom left: learned sector gates.  Bottom right:
LC-core teacher accuracy.}
\label{fig:synthetic}
\end{figure}

\subsection{Natural-language boundary test}

In the byte-level language runs, the dominant behavior is affine/recency
selection.  We train small byte-level language models at short context and
evaluate them at long context, including 32768-token stress settings.
NTK-style RoPE scaling plus affine recency gives the lowest 32768-token loss on
Tiny Shakespeare, WikiText-2, and War and Peace.  High-order FJ/LC mass appears
mainly as a diagnostic correction; affine recency is the leading source of loss
improvement in this slice.

\subsection{Music-token transfer}

Music-token streams allocate differently from byte-level language.  On MAESTRO
MIDI controls/pedal streams and two MusicNet reference-MIDI selectors, LC-affine
variants remain competitive or best at 32768 tokens, and the learned high-order
mass is small but consistently nonzero.  This pattern is consistent with motif
returns, rhythm envelopes, phrase timing, and long-range repetition.  The exact
PJ-rotary baseline provides a feature-transform closure control: it fits near the
training scale but degrades sharply at 32768 tokens, separating representation
correctness from stable long-context usefulness.
Table~\ref{tab:natural-summary} reports the main 32768-token contrasts.

\begin{table}[H]
\centering
\caption{Natural task summaries at 32768 tokens.  Panel~(a) reports the
language allocation; Panel~(b) reports music-token allocation.  Entries are
mean \(\pm\) standard deviation over three model seeds; lower validation
cross-entropy is better.}
\label{tab:natural-summary}
\scriptsize
\begin{subtable}{\linewidth}
\centering
\caption{Language boundary.  Models are trained at context length 1024.  The
second-best column reports the next-best scaled-RoPE-plus-affine variant.}
\label{tab:language-boundary}
\setlength{\tabcolsep}{2pt}
\begin{tabular}{lrrlclc}
\toprule
Dataset & Train & Eval & Best & Loss & 2nd scaled-aff. & Loss \\
\midrule
Tiny Shakespeare & 1024 & 32768 & NTK+affine & \(2.182\pm0.123\) & YaRN+affine & \(2.412\pm0.037\) \\
WikiText-2 & 1024 & 32768 & NTK+affine & \(2.048\pm0.037\) & YaRN+affine & \(2.267\pm0.071\) \\
War and Peace & 1024 & 32768 & NTK+affine & \(1.876\pm0.133\) & YaRN+affine & \(2.106\pm0.162\) \\
\bottomrule
\end{tabular}
\end{subtable}

\vspace{0.65em}

\begin{subtable}{\linewidth}
\centering
\caption{Music-token allocation.  Models are trained at context length 512.
The high-order mass column reports the total high-order mass of the LC-affine
row.}
\label{tab:music-positive}
\setlength{\tabcolsep}{2pt}
\begin{tabularx}{\linewidth}{Xrrcccc}
\toprule
Dataset / selector & Train & Eval & LC-affine & Full PJ & NTK+affine & High-order mass \\
\midrule
MAESTRO controls & 512 & 32768 & \(0.921\pm0.037\) & \(0.929\pm0.118\) & \(1.518\pm0.162\) & \(0.038\pm0.009\) \\
MusicNet selector A & 512 & 32768 & \(0.958\pm0.110\) & \(1.392\pm0.411\) & \(1.629\pm0.420\) & \(0.042\pm0.009\) \\
MusicNet selector B & 512 & 32768 & \(0.884\pm0.207\) & \(0.932\pm0.232\) & \(2.060\pm0.650\) & \(0.042\pm0.007\) \\
\bottomrule
\end{tabularx}
\end{subtable}
\end{table}

\begin{figure}[t]
\centering
\includegraphics[width=0.48\linewidth]{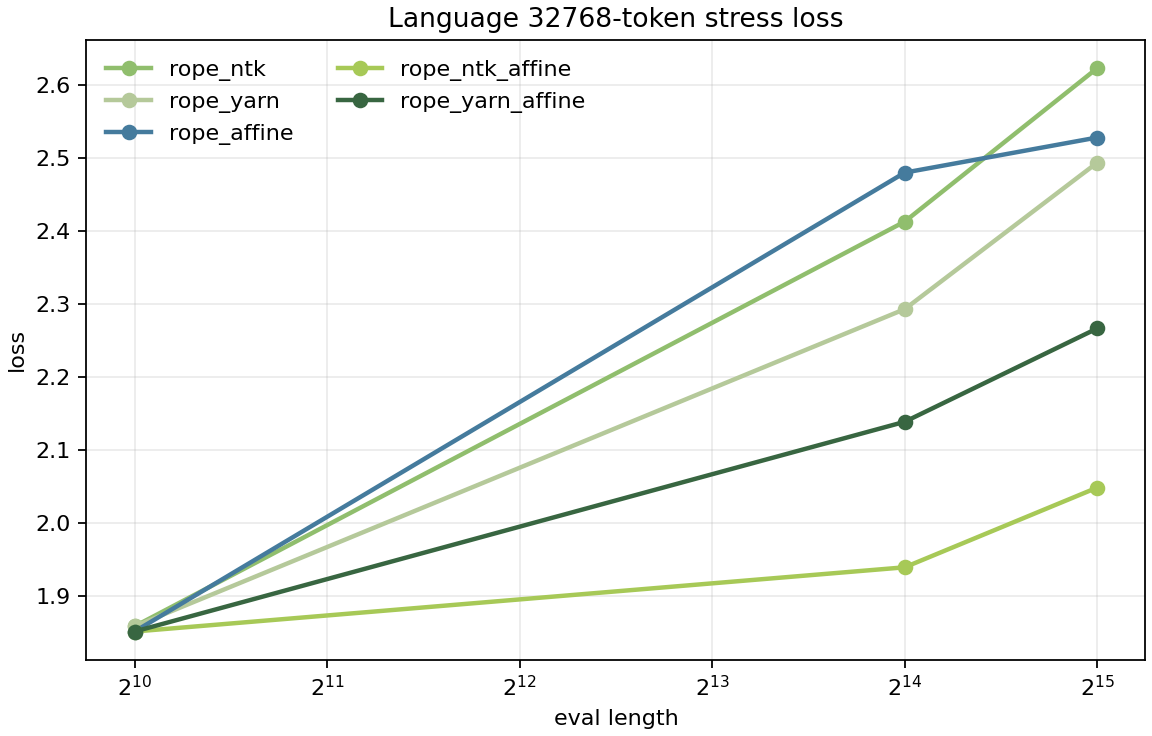}
\includegraphics[width=0.48\linewidth]{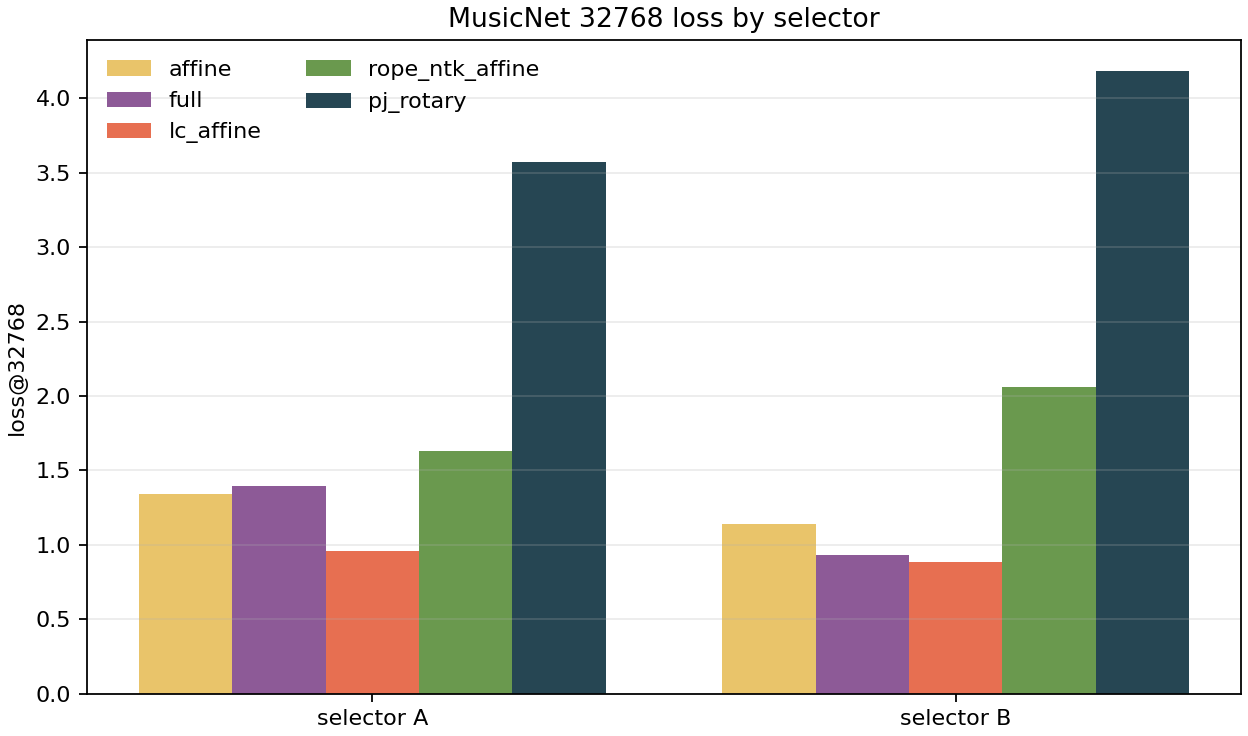}
\includegraphics[width=0.48\linewidth]{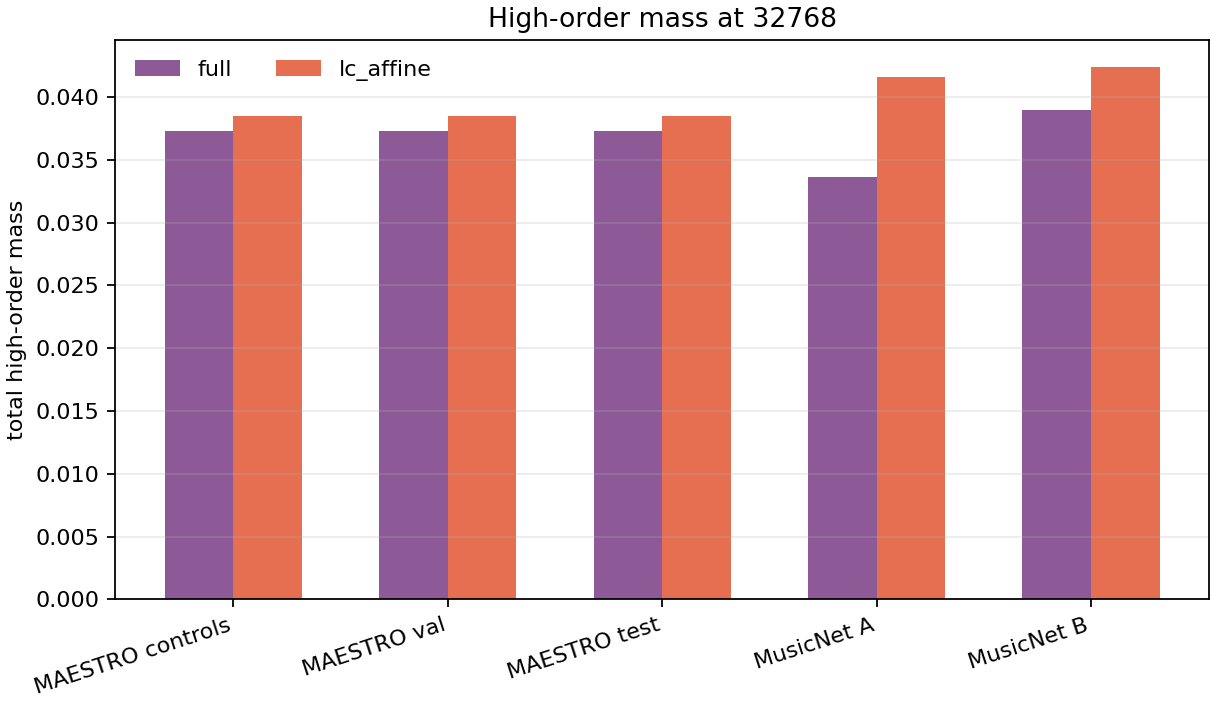}
\includegraphics[width=0.48\linewidth]{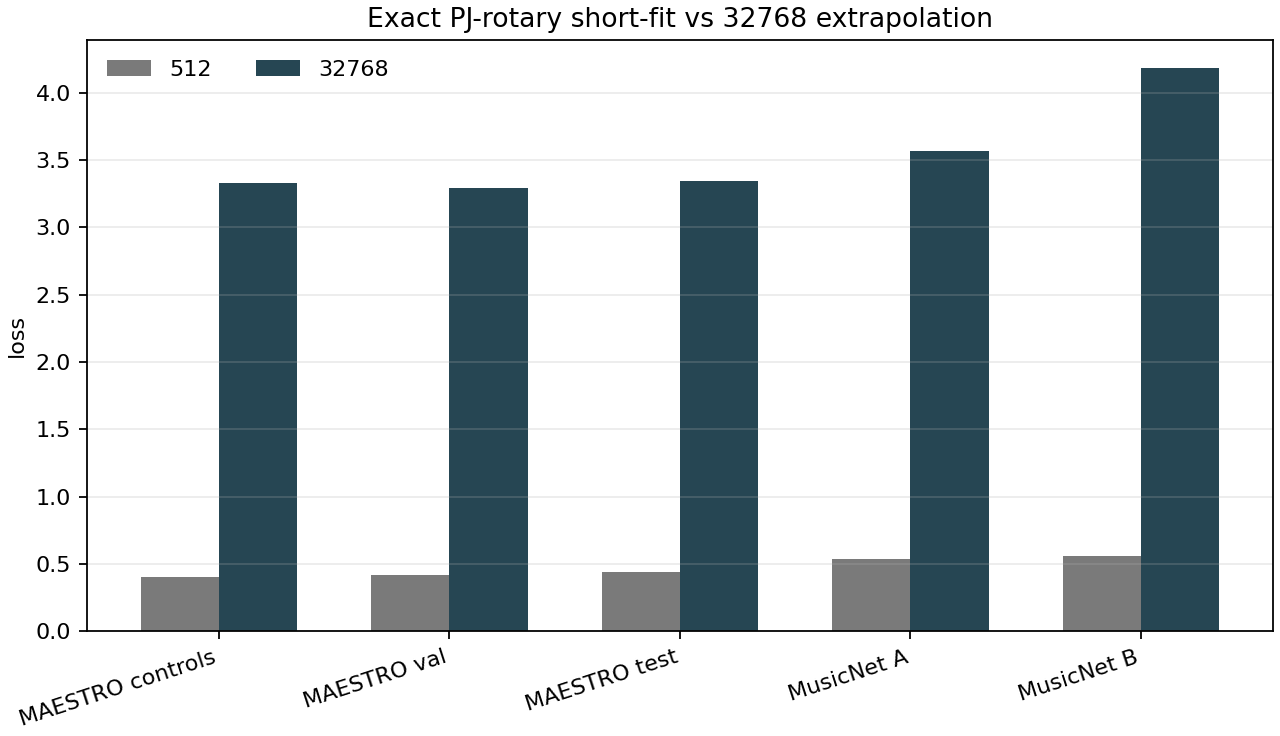}
\caption{Natural task contrast.  Language concentrates on recency/affine
behavior, while MusicNet/MAESTRO MIDI-token streams keep LC/affine variants
strong with measurable high-order corrections.  The exact PJ-rotary panel is a
feature-transform closure control: feature-transform correctness does not by
itself guarantee stable extrapolation.}
\label{fig:natural}
\end{figure}

\clearpage
\subsection{GRAPE special-case reruns}

GRAPE is the closest existing group-action framework to \pjrope{}.  We
therefore include three exact special-case controls: GRAPE-M/RoPE,
GRAPE-A/ALiBi, and GRAPE-M+A/RoPE+ALiBi.  These controls provide the closest
standard multiplicative-rotation and additive-recency reference point.

The fixed-projection comparison in Appendix~\ref{app:grape} isolates primitive
containment.  GRAPE-M+A covers separate phase and affine directions, while
PJ-FJ contains phase-modulated distance terms such as
\((d/L)\cos\omega d\) and \((d/L)^2\cos\omega d\).  The trainable reruns then
show small-model finite-budget behavior: GRAPE-M+A is a strong natural-task
control, PJ and LC-PJ variants are competitive in some rows, and the best
method depends on domain and training regime.  This separates primitive
containment from trainable optimization behavior.

\subsection{Light-cone stabilization}

The LC experiments test the stabilizer directly.  Raw high-order coordinates
can create extreme transformed-key norms and logit scales at 32768 tokens.  LC
variants replace raw distance with compactified phase and saturating amplitude,
so the relevant measurements are not only loss but also Q/K proxy, effective
support, phase span, cache/logit scale, quantization error, and retrieval
resolution.

The observed pattern matches the theoretical expectation.  LC variants bound
coordinate and cache scale, while the same compression reduces far-range
resolution, visible in int4 and hard-negative retrieval probes.  The LC
measurements show the expected exchange: coordinate and cache scale are
bounded, while far-range phase span and hard-negative retrieval resolution
decrease.  Table~\ref{tab:lc-summary} gives the core stability and resolution
metrics.

For clarity, the LC table uses the same stability proxy as Figure~\ref{fig:lc}.
``QK proxy'' is the final-lag query/key norm proxy from the feature-scale sweep,
and ``cache logit std'' is the standard deviation of the cached logit proxy at
the evaluation length.  Let \(W_{\mathrm{far}}\) denote the far evaluation
bucket and let \(\psi(d)\) be the phase coordinate used by a variant.  We define
\[
  \mathrm{phase\ ratio}
  =
  \frac{\max_{d\in W_{\mathrm{far}}}\psi(d)-\min_{d\in W_{\mathrm{far}}}\psi(d)}
       {\max_{d\in W_{\mathrm{far}}}d-\min_{d\in W_{\mathrm{far}}}d},
\]
normalized so that raw and scaled coordinates equal \(1\).  The far span is the
corresponding frequency-weighted phase span,
\[
  \max_{d\in W_{\mathrm{far}}}\omega\psi(d)
  -
  \min_{d\in W_{\mathrm{far}}}\omega\psi(d).
\]

\begin{table}[H]
\centering
\caption{LC stability and phase-resolution diagnostics at 32768 tokens.  The
wrong-scale control is numerically stable but loses usable phase span.}
\label{tab:lc-summary}
\scriptsize
\setlength{\tabcolsep}{3pt}
\begin{tabularx}{\linewidth}{lrrrrrrX}
\toprule
Variant & QK proxy & Cache logit std & Phase ratio & Far span & Int8 top1 & Int4 top1 & Reading \\
\midrule
raw & \(5.86{\times}10^{12}\) & \(1.78{\times}10^{24}\) & 1.000 & 1708.92 &
0.992 & 0.594 & Scale explodes; attention collapses despite high phase span. \\
scaled & \(2.24{\times}10^{2}\) & \(3.97{\times}10^{3}\) & 1.000 & 1708.92 &
1.000 & 0.891 & Reduces raw growth but leaves large cache/logit scale. \\
lc & \(2.00\) & \(5.16{\times}10^{-1}\) & 0.031 & 73.96 &
1.000 & 0.109 & Stabilizes scale; int4 retrieval exposes resolution cost. \\
lc\_wrong\_scale & \(2.00\) & \(1.21{\times}10^{-1}\) & \(3.05{\times}10^{-5}\) & 0.07 &
0.031 & 0.000 & Stable norm but collapsed phase span. \\
\bottomrule
\end{tabularx}
\end{table}

\clearpage
\begin{figure}[H]
\centering
\includegraphics[width=0.48\linewidth]{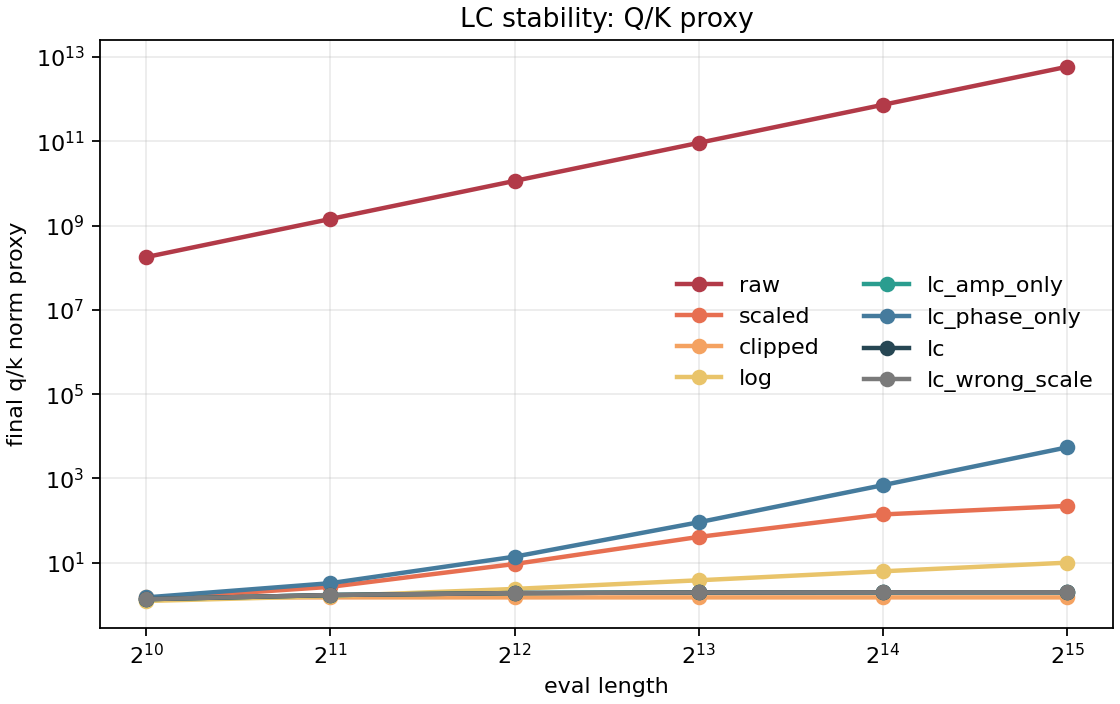}
\includegraphics[width=0.48\linewidth]{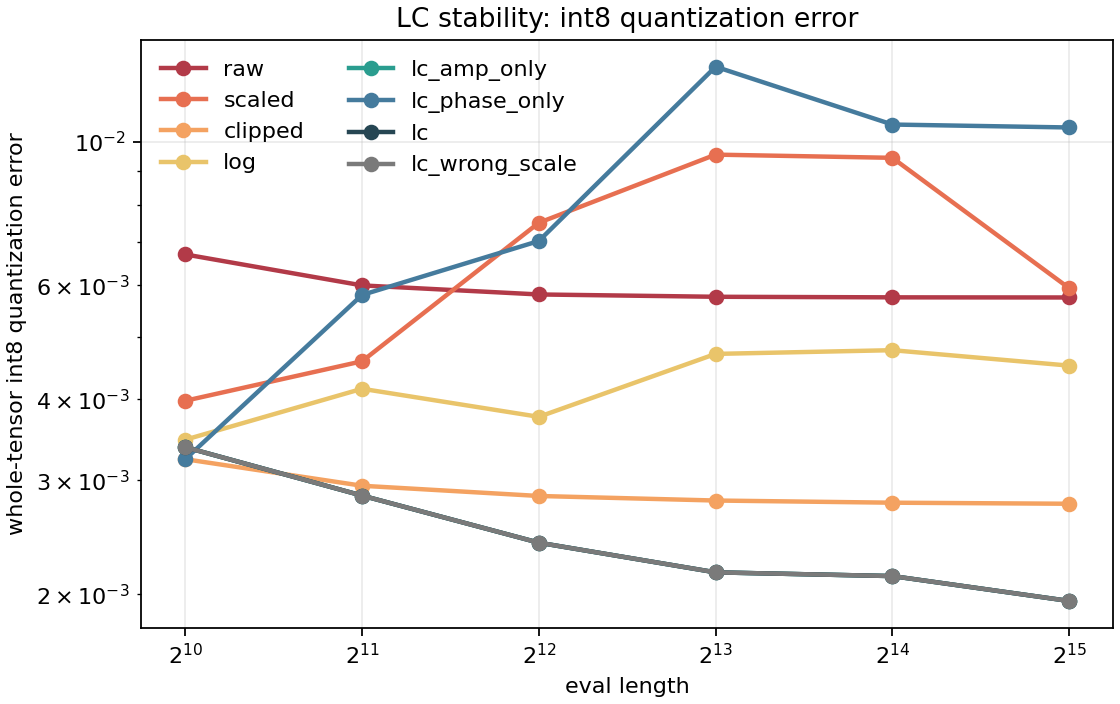}
\includegraphics[width=0.48\linewidth]{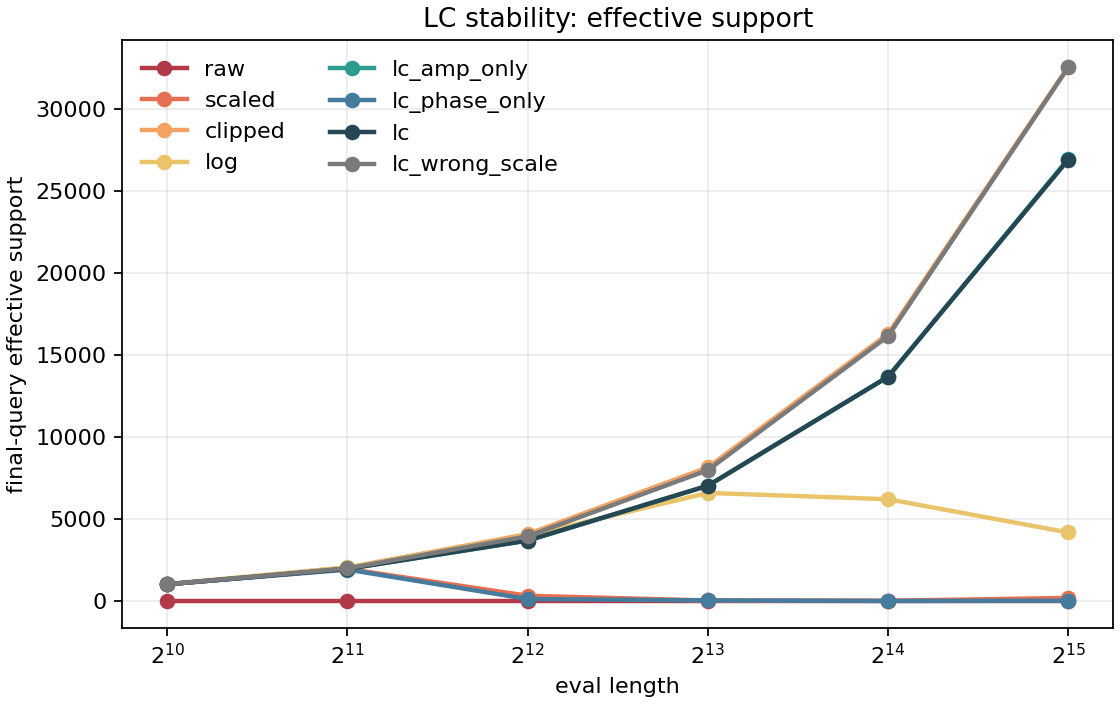}
\includegraphics[width=0.48\linewidth]{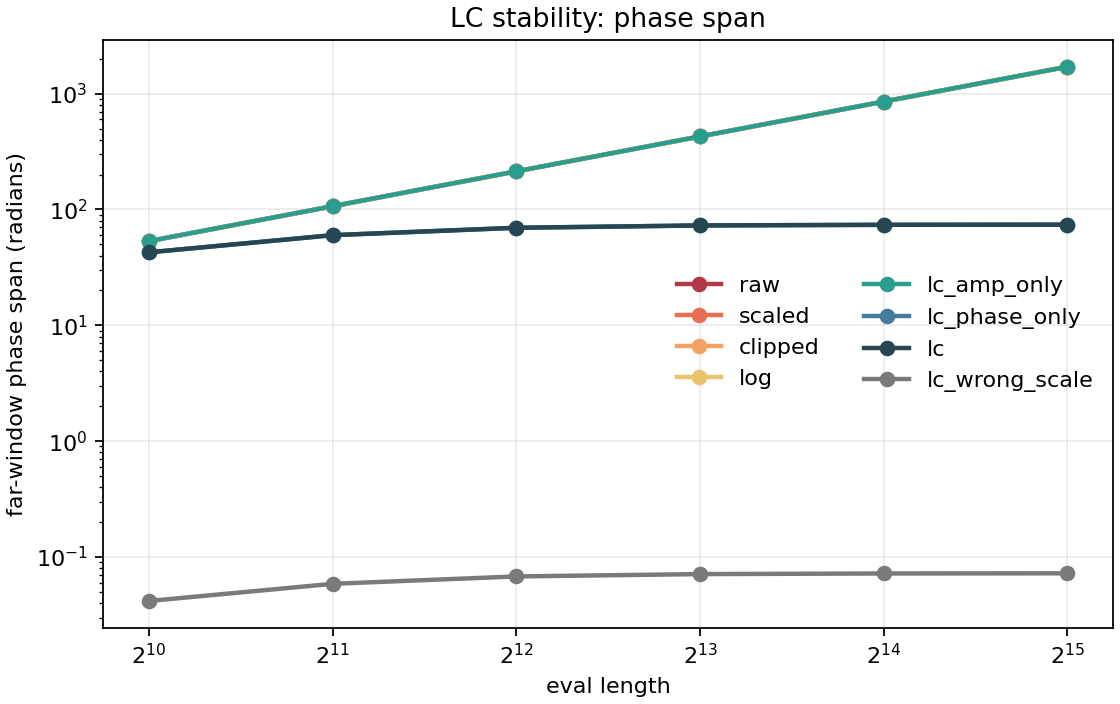}
\caption{LC stability tradeoff.  LC variants bound scale and cache pressure,
but the same compactification exposes a phase-resolution cost.}
\label{fig:lc}
\end{figure}

\paragraph{Code and reproducibility.}
We provide the minimal implementation and experiment scripts needed to rerun the
reported studies and regenerate figures and tables locally.  Generated tables,
figures, run logs, large checkpoints, and raw corpora are excluded.  The
reproducibility repository is available at
\url{https://github.com/ybzhang-nxu/Poincare_Rope}.

\section{Discussion and Scope}

The experiments give a sector-selection picture.  Controlled kernels and
adaptive probes recover the designed Fourier, finite-jet, affine, and LC
regions.  Synthetic sequence tasks show that these regions can be used inside
trainable attention when the positional branch carries the teacher signal.
Byte-level language runs concentrate on affine/recency behavior under the
small-model, short-context training regime used here.  Symbolic music-token
streams show stronger LC/affine behavior with small but measurable high-order
corrections.  LC diagnostics then explain the stabilization mechanism by
measuring the accompanying loss of far-range phase resolution.

\paragraph{Relation to common relative-position regimes.}
The observed allocations are consistent with the broader relative-position
literature.  In the byte-level language runs, the strongest behavior is
affine/recency selection, matching the empirical role of bias-centric and
scaled-RoPE methods in long-context language modeling.  The fixed-kernel and
controlled synthetic tasks occupy the Fourier--jet region, where
distance-modulated phase terms are present by construction.  The music-token
runs give an intermediate allocation: LC/affine variants remain strong, while
high-order mass is small but measurable.  The LC diagnostics isolate the
corresponding implementation tradeoff: compactification controls high-order
scale and cache pressure with a reduction in far-range phase resolution.

The strong NTK+affine language result can be read through the
difference-module viewpoint.  NTK-aware RoPE and adjusted-base-frequency
variants change the RoPE frequency grid
\cite{Bloc97NTK2023,Xiong2023LongContextScaling,Peng2024,Liu2024ScalingRoPE},
hence moving the simple Fourier roots \(z_\ell=e^{i\omega_\ell}\) along a
spectral flow.  The tangent directions of this flow are first Fourier jets
\(d z_\ell^d\), while higher Taylor directions generate higher jets
\(d^r z_\ell^d\).  The affine branch supplies the analogous zero-frequency jet
at \(z=1\).  Appendix~\ref{app:ntk_diff_module} makes this relation explicit.

\begin{table}[t]
\centering
\caption{Relative-position families and \pjrope{} sectors.  The table
positions \pjrope{} as a primitive-family formulation alongside existing
relative-position mechanisms.}
\label{tab:relative-families}
\small
\setlength{\tabcolsep}{3pt}
\begin{tabularx}{\linewidth}{l>{\raggedright\arraybackslash}X>{\raggedright\arraybackslash}X>{\raggedright\arraybackslash}X}
\toprule
Existing family & Typical examples & Primitive & \pjrope{} relation \\
\midrule
Local / bucketed RPE & Shaw, Transformer-XL, T5-style bias~\cite{Shaw2018,Dai2019,Raffel2020} & learned offset or bucket & reference family \\
Bias-centric recency & ALiBi, T5 relative bias~\cite{Press2022,Raffel2020} & scalar function of \(d\) & affine sector \\
Kernelized / functional bias & KERPLE, FIRE, MEP~\cite{Chi2022,Li2024,Gao2024MEP} & smooth scalar kernels & adjacent scalar-kernel family \\
Rotary phase & RoPE, XPos, PI, YaRN, LongRoPE~\cite{Su2021,Sun2023,Chen2023,Peng2024,Ding2024} & Fourier phase & Fourier sector \\
Non-semisimple phase & Jordan-RoPE~\cite{JordanRoPE2026} & \(d^r e^{i\omega d}\) & finite-jet sector \\
Hyperbolic / log distance & HyPE and nonlinear distance biases~\cite{Angelotti2023HyPE} & nonlinear distance coordinate & related to LC coordinate compression \\
Efficient RPB & FastRPB~\cite{Zubkov2022FastRPB} & Toeplitz / FFT acceleration & future implementation path for PJ-bias \\
\bottomrule
\end{tabularx}
\end{table}

Since scalar PJ-bias kernels are lag-indexed kernels, FFT-style relative-bias
acceleration such as FastRPB is a natural implementation direction for scaling
the bias path~\cite{Zubkov2022FastRPB}.

The dataset interpretation is also specific.  The MusicNet result is a
reference-MIDI token result, not an audio-label MusicNet result.  The evidence
supports a music-token allocation pattern for structured symbolic streams,
and does not address audio transcription or acoustic classification.

\paragraph{Feature correctness versus long-context stability.}
The exact PJ-rotary implementation is a controlled feature-transform contrast.
It verifies that the feature-transform path can realize
RoPE/Jordan-RoPE-style relative actions, but Figure~\ref{fig:natural} shows a
separate stability requirement at 32768 tokens: loss rises sharply when the same
high-order feature transform is evaluated far beyond the 512-token training
scale.  This contrast separates representation closure from long-context
numerical stability.  The scalar LC path serves as the stabilized chart for
long-context high-order coordinates.

\paragraph{Scale and generality.}
The trainable evidence uses two-layer, 96-dimensional small Transformers and
large extrapolation ratios: language rows train at 1024 bytes and evaluate up
to 32768 bytes, while music-token rows train at 512 bytes and evaluate up to
32768 bytes.  These settings are deliberately stress tests of positional
mechanisms.  They do not establish that the same allocation will hold unchanged
in billion-parameter models, longer training schedules, or subword-tokenized
language models.  Main tables now report seed standard deviations where
available, but several conclusions still rely on small numbers of seeds and
single long-context evaluation batches.  Larger-scale validation is therefore
needed before treating the observed allocations as deployment rules.

\section{Conclusion}

\pjrope{} reframes relative positional representation as a learnable
Fourier--Jet--Affine space over the lag axis.  The difference-module view
explains why these mechanisms belong to one structure: RoPE is a simple
Fourier root of the lag shift, Jordan-RoPE and high-order PJ coordinates are
repeated-root Fourier jets, and ALiBi-like recency is the repeated unit-root
affine direction.  NTK-aware RoPE scaling fits this same picture as a spectral
flow of simple Fourier roots, whose tangent directions are first Fourier jets.
Thus \pjrope{} makes explicit a finite family of jet directions that are
implicit in local spectral deformations of RoPE.

Empirically, the paper uses this space as a diagnostic rather than as a single
universal replacement for existing positional encodings.  Controlled kernels
and synthetic teachers recover their intended sectors; small byte-level
language models prefer NTK-aware scaling plus affine recency; symbolic
music-token streams keep LC/affine variants strong with small but measurable
high-order corrections.  LC/rapidity coordinates stabilize the high-order
regime, but the same compactification reduces far-range phase resolution.
Stable long-context use therefore depends on a task-dependent allocation
between affine recency, Fourier--jet structure, and LC compactification.

\clearpage
\appendix
\section{Detailed Comparison with GRAPE}
\label{app:grape}

\subsection{Scope of the comparison}

GRAPE provides the closest existing group-action unification framework for
RoPE-like multiplicative rotations and ALiBi-like additive biases
\cite{Zhang2026GRAPE}.  The comparison locates the overlap and the difference:
GRAPE emphasizes exact group-action laws and cacheable relative actions, while
\pjrope{} emphasizes the non-semisimple Fourier--jet sector, adaptive sector
diagnostics, and LC/rapidity stabilization.

The baselines in this appendix are exact special cases of the GRAPE framework,
not implementations of full learned GRAPE-M, GRAPE-A, or GRAPE-AP.  We compare
GRAPE-M special case / RoPE, GRAPE-A special case / ALiBi, and GRAPE-M+A
special case / RoPE+ALiBi as controlled primitive bases.  The comparison is
about primitive function spaces, sector selection, and high-order stability.

\subsection{What GRAPE covers}

GRAPE formulates positional encoding as group actions with exact relative laws.
Multiplicative GRAPE represents positions through norm-preserving rotations
\(G(n)=\exp(n\omega L)\in SO(d)\), with rank-two skew generators.  RoPE is
recovered when the rotation planes are canonical coordinate pairs with a chosen
frequency spectrum.  Additive GRAPE realizes additive logit biases through
low-rank unipotent lifts in a larger linear group, recovering ALiBi-like
additive slopes as exact special cases.  These properties are important:
exact relative composition, norm-preserving multiplicative actions, additive
unipotent lifts, and streaming cacheability are central strengths of GRAPE.

\pjrope{} is complementary on a different axis.  It asks what changes when the
rotary Fourier character itself is made non-semisimple, producing finite jets
\(d^r e^{i\omega d}\), and how those high-order coordinates can be diagnosed
and stabilized at long context.

\subsection{Algebraic distinction: where the nilpotent lives}

In the additive GRAPE special case, the nilpotent lives in a unipotent lift that
generates an additive bias.  A schematic form is
\[
  G_{\mathrm{add}}(n)=I+n\omega A,\qquad A^2=0.
\]
This is the correct algebraic home for ALiBi-like additive slopes.

In the Fourier--jet sector of \pjrope{}, the nilpotent instead lives inside the
complex rotary eigenvalue block:
\[
  J=(-\gamma+i\omega)I+\eta N,\qquad N^m=0.
\]
Exponentiating gives
\[
  e^{dJ}
  =
  e^{(-\gamma+i\omega)d}
  \sum_{r=0}^{m-1}\frac{(\eta d)^r}{r!}N^r.
\]
Thus the primitive modes include
\[
  e^{i\omega d},\quad d e^{i\omega d},\quad d^2 e^{i\omega d},\ldots.
\]
The key difference is that GRAPE-M+A special cases combine phase and distance
as separate primitives, whereas PJ-FJ makes distance-modulated phase a
primitive.

\begin{lemma}[Direct-sum phase/affine features do not contain Fourier jets]
\label{lem:grape-direct-sum}
For \(\omega\not\equiv 0,\pi\), the function \(d e^{i\omega d}\) is not
contained in the finite span
\[
  \operatorname{span}\{1,d,e^{i\omega d},e^{-i\omega d}\}
\]
over infinitely many integer lags.
\end{lemma}

\begin{proof}[Proof sketch]
This follows from the linear independence of exponential-polynomial sequences
with distinct characteristic roots.  The term \(d e^{i\omega d}\) corresponds
to a repeated root at \(e^{i\omega}\), while the direct-sum phase/affine basis
only contains simple roots at \(e^{\pm i\omega}\) and the affine root at \(1\).
This is the special case of the repeated-root description in
Appendix~\ref{app:difference-modules}.
\end{proof}

\subsection{Conceptual comparison table}

\begin{table}[H]
\centering
\caption{Conceptual comparison between GRAPE and \pjrope{}.  GRAPE emphasizes
exact group-action unification and cacheable relative laws; \pjrope{} emphasizes
non-semisimple Fourier jets, adaptive sector diagnostics, and LC stabilization.}
\label{tab:grape-conceptual}
\small
\begin{tabularx}{\linewidth}{l>{\raggedright\arraybackslash}X>{\raggedright\arraybackslash}X}
\toprule
Aspect & GRAPE & \pjrope{} \\
\midrule
Main object & group-action positional encoding & relative-position module space \\
Homogeneous side & \(SO(d)\) rotations & Fourier characters and finite jets \\
RoPE recovery & canonical rank-two skew rotations & order-zero Fourier sector \\
Nilpotent structure & unipotent lift for additive bias & defective complex Fourier block \\
Primitive modes & rotations and additive slopes & \(d^r e^{i\omega d}\) \\
ALiBi & exact unipotent special case & affine recency sector \\
Exact relative law & central design criterion & exact for PJ-rotary, scalar for PJ-bias \\
Stability emphasis & norm-preserving rotations / streaming laws & LC compactification of high-order jets \\
Empirical emphasis & deployable PE framework & sector diagnostics and task-dependent selection \\
\bottomrule
\end{tabularx}
\end{table}

\begin{table}[H]
\centering
\caption{Exactness and primitive-mode sanity table.  GRAPE special-case rows cover exact group-action laws; PJ rows identify the Fourier-jet and LC axes used in this paper.}
\label{tab:grape-exactness}
\small
\begin{tabularx}{\linewidth}{>{\raggedright\arraybackslash}Xccccc}
\toprule
Method & Exact relative law & Norm preserving & Additive recency & \(d^r e^{i\omega d}\) & LC bounded \\
\midrule
GRAPE-M special case & yes & yes & no & no & no \\
GRAPE-A special case & yes & n/a & yes & no & no \\
GRAPE-M+A special case & yes, separate & partial & yes & no & no \\
PJ-rotary exact & yes & no generally & no & yes & no \\
PJ-bias & scalar kernel & n/a & yes & yes & optional \\
LC-PJ bias & scalar kernel & n/a & yes & compactified jets & yes \\
\bottomrule
\end{tabularx}
\end{table}

\subsection{Controlled primitive-basis evidence}

The fixed projection experiment makes the algebraic distinction measurable.  We
fit each target \(y(d)\) by least squares over a fixed basis \(B(d)c\).  All
rows use the same frequency grid.  PJ-FJ variants add jet orders at that
frequency; the experiment is a primitive-containment probe.

The important targets are \(\cos\omega d\), \(-d/L\),
\((d/L)\cos\omega d\), and \((d/L)^2\cos\omega d\).  The direct-sum
GRAPE-M+A special case covers phase and affine targets.  Phase-modulated
distance appears as an explicit primitive in PJ-FJ order one and two.

\begin{table}[H]
\centering
\caption{GRAPE special-case and PJ primitive-basis projections.  All rows use the same frequency grid; PJ-FJ variants add jet orders at that frequency.  The table reports primitive containment through fixed projection.  Values are \(R^2\) at the longest evaluation length; em dashes mark failed extrapolations with \(R^2<-10\).}
\label{tab:grape-fixed-projection}
\small
\setlength{\tabcolsep}{3pt}
\begin{tabular}{lrrrrrr}
\toprule
Target & GRAPE-M/RoPE & GRAPE-A/ALiBi & GRAPE-M+A & PJ-FJ R1 & PJ-FJ R2 & PJ-LC \\
\midrule
\(\cos\omega d\) & 1.000 & -0.022 & 1.000 & 1.000 & 1.000 & \textemdash \\
\(-d/L\) & -2.815 & 1.000 & 1.000 & -2.815 & \textemdash & -2.817 \\
\((d/L)\cos\omega d\) & 0.046 & -0.000 & 0.046 & 1.000 & 1.000 & -0.044 \\
\((d/L)^2\cos\omega d\) & 0.001 & 0.000 & 0.001 & 0.076 & 1.000 & -0.000 \\
LC core & -0.046 & -0.051 & -0.054 & \textemdash & \textemdash & 1.000 \\
\bottomrule
\end{tabular}
\end{table}

\subsection{Cross-task stress reruns with GRAPE special-case controls}

Table~\ref{tab:grape-task-reruns} reports GRAPE special-case reruns under the
same small trainable setting.  Natural-task rows use the 8192-token maximum
length from this appendix sweep; the main natural-task table uses the
32768-token budget.  Table~\ref{tab:grape-fixed-projection} tests primitive
containment by projection; these rows test finite-budget optimization and
extrapolation.  In the signed first-jet row, GRAPE-M/RoPE outperforms PJ
variants at 384 tokens, although it does not contain the Fourier-jet primitive.
This indicates that a phase-only surrogate can correlate with the thresholded
teacher labels over the sampled window, while the learned jet branch is less
stable under this extrapolation setting.  Thus
Table~\ref{tab:grape-fixed-projection} and
Table~\ref{tab:grape-task-reruns} answer different questions: primitive
containment versus finite-budget trainability.

\begin{table}[H]
\centering
\caption{Cross-task stress reruns with GRAPE special-case controls.  Rows use restricted GRAPE-M/A exact special-case controls.  Synthetic rows report accuracy over three seeds; language and music rows report validation cross-entropy over two seeds at the maximum length available in this appendix sweep, 8192 tokens for natural-task rows.  The table reports finite-budget trainability and extrapolation; Table~\ref{tab:grape-fixed-projection} reports primitive containment.  The PJ reference column is PJ-FJ for synthetic rows and LC-PJ+A for natural rows.}
\label{tab:grape-task-reruns}
\scriptsize
\begin{tabularx}{\linewidth}{lXrrrrrr}
\toprule
Task & Target / run & Eval & GRAPE-M & GRAPE-A & GRAPE-M+A & PJ full & PJ ref \\
\midrule
Language & Tiny Shakespeare & 8192 & 3.614 & 2.490 & 2.498 & 2.491 & 2.491 \\
Language & WikiText-2 & 8192 & 3.070 & 2.417 & 2.559 & 2.406 & 2.402 \\
Music & MAESTRO controls & 8192 & 2.335 & 1.163 & 1.138 & 1.172 & 1.144 \\
Music & Motif-rich music & 8192 & 2.200 & 1.108 & 0.479 & 1.248 & 1.147 \\
Music & MusicNet selector A & 8192 & 2.321 & 1.350 & 1.105 & 1.334 & 1.380 \\
Synthetic & affine & 384 & 0.944 & 0.885 & 0.899 & 0.896 & 0.934 \\
Synthetic & first jet & 384 & 0.647 & 0.501 & 0.591 & 0.447 & 0.465 \\
Synthetic & phase & 384 & 0.760 & 0.495 & 0.753 & 0.786 & 0.729 \\
\bottomrule
\end{tabularx}
\end{table}

\subsection{Summary of the comparison}

Overall, GRAPE and \pjrope{} overlap on the unification axis but emphasize
different coordinates.  GRAPE emphasizes exact group-action laws, norm-preserving
multiplicative rotations, additive unipotent lifts, and cacheable relative
actions.  \pjrope{} emphasizes non-semisimple Fourier jets, adaptive sector
diagnostics, and LC stabilization of high-order behavior.

\section{Constant-coefficient Difference Modules and Fourier Jets}
\label{app:difference-modules}

This appendix makes explicit the difference-module viewpoint used in the main
text.  It is not a new experimental claim.  Its purpose is to put RoPE,
Jordan-RoPE, ALiBi, and the Fourier--jet part of \pjrope{} into the standard
language of shift operators and repeated characteristic roots.

\subsection{Lag-shift motivation for difference modules}
\label{app:diff_lag_shift_motivation}

The difference-module viewpoint starts from a simple property of relative
attention.  A relative attention kernel is a response function on the lag axis.
If a query at position \(i\) attends to a key at position \(j\), the relative
variable is
\[
  d=i-j,\qquad d\ge 0.
\]
After passing from absolute positions to relative positions, the natural
operation is to move one step farther along the lag axis:
\[
  d\mapsto d+1.
\]
This operation is the lag-shift operator
\[
  (Ef)(d)=f(d+1).
\]
Thus \(E\) is not merely an auxiliary algebraic symbol.  It records how a
relative-position response changes when the key is moved one token farther
from the query.

A finite structured positional mechanism should describe this one-step lag
evolution with finitely many internal coordinates.  Equivalently, the shifted
responses
\[
  K,\ EK,\ E^2K,\ldots
\]
should span a finite-dimensional space.  This means that there exists a
nonzero polynomial \(P\) such that
\[
  P(E)K=0.
\]
This is the constant-coefficient difference-module condition.

Equivalently, \(K\) is a finite exponential-polynomial lag response.  The
transfer-operator form of the same statement is given in
Section~\ref{app:diff_finite_representation}.

The spectral type of \(E\) therefore classifies relative-position primitives.
A simple unit-modulus root
\[
  z=e^{i\omega}
\]
gives the RoPE/Fourier character
\[
  K(d)=z^d=e^{i\omega d},
\]
with \((E-z)K=0\), so each one-step increase in lag advances the phase by a
fixed angle.  A root
inside the unit disk gives a damped oscillatory mode and models long-distance
correlation decay.  A repeated nonzero root gives polynomially modulated
Fourier modes
\[
  d^r z^d,
\]
which are the Fourier jets used by Jordan-RoPE and \pjrope{}.  Finally, the
repeated unit root \(z=1\) gives
\[
  1,\ d,
\]
which is the affine recency direction used by ALiBi-like biases.
Section~\ref{app:diff_root_stability} separates out the corresponding
stability and correlation-decay consequences of root location.

This viewpoint also clarifies what is relaxed relative to ordinary RoPE.  RoPE
is the unitary, semisimple, cyclic case of lag-shift dynamics: increasing the
lag only rotates phase and does not change amplitude.  Causal long-context
attention is not a periodic circle; it lives on a one-sided lag axis and often
requires recency, decay, and non-semisimple corrections.  Difference modules
are therefore a natural non-unitary and non-semisimple extension of the RoPE
Fourier picture.

\subsection{Repeated roots and finite jets}

The natural polynomial basis for repeated roots of the shift operator is the
binomial basis
\[
  \phi_{z,r}(d)=\binom{d}{r}z^d,
  \qquad r\ge 0.
\]
We set \(\phi_{z,-1}=0\).

\begin{proposition}[Repeated roots generate finite jets]
\label{prop:difference-repeated-roots}
Let \(E\) be the shift operator and define
\[
  \phi_{z,r}(d)=\binom{d}{r}z^d.
\]
Then
\[
  (E-z)\phi_{z,r}=z\phi_{z,r-1},
\]
and therefore
\[
  \phi_{z,r}\in \ker(E-z)^{r+1}.
\]
Thus a root \(z\) of multiplicity \(m\) generates the modes
\[
  z^d,\binom d1 z^d,\ldots,\binom d{m-1}z^d.
\]
\end{proposition}

\begin{proof}
Using \(\binom{d+1}{r}-\binom{d}{r}=\binom{d}{r-1}\), we have
\[
  (E-z)\phi_{z,r}(d)
  =
  \binom{d+1}{r}z^{d+1}
  -
  z\binom{d}{r}z^d
  =
  z\binom{d}{r-1}z^d
  =
  z\phi_{z,r-1}(d).
\]
Iterating the identity gives \((E-z)^{r+1}\phi_{z,r}=0\).
\end{proof}

The main text often writes jet factors as \((d/L)^r z^d\).  This is only a
different normalization.  For any fixed maximum order, the monomial basis
\(\{d^r\}\) and the binomial basis \(\{\binom{d}{r}\}\) are related by an
invertible triangular change of basis, so they span the same finite
polynomial-jet space.

\subsection{RoPE, Jordan-RoPE, and ALiBi as roots}

The root interpretation gives a compact dictionary for the positional
primitives in this paper.

\begin{table}[H]
\centering
\caption{Difference-module interpretation of common relative-position primitives.}
\label{tab:difference-root-dictionary}
\small
\begin{tabularx}{\linewidth}{>{\raggedright\arraybackslash}X>{\raggedright\arraybackslash}X>{\raggedright\arraybackslash}X}
\toprule
Mechanism & Difference root type & Primitive modes \\
\midrule
RoPE & simple nonzero unit root \(z=e^{i\omega}\) & \(e^{i\omega d}\) \\
Damped RoPE & simple interior root \(z=e^{-c/L+i\omega}\) & \(e^{-cd/L}e^{i\omega d}\) \\
Jordan-RoPE & repeated nonzero root \(z=e^{i\omega}\) & \(d^r e^{i\omega d}\) \\
ALiBi & repeated unit root \(z=1\) & \(1,d\) \\
\pjrope{} & repeated Fourier roots plus repeated unit root & \(d^r e^{i\omega d},1,d\) \\
\bottomrule
\end{tabularx}
\end{table}

In this notation, RoPE satisfies \((E-e^{i\omega})K=0\).  First-order
Jordan-RoPE satisfies \((E-e^{i\omega})^2K=0\).  ALiBi's affine direction is the
unit-root jet, since \((E-1)^2 d=0\).  A first-order complex PJ kernel can be
viewed as a solution of
\[
  (E-z)^2(E-1)^2K=0,
  \qquad z=e^{i\omega}.
\]
For real-valued kernels, the conjugate root must also be included:
\[
  (E-z)^2(E-\overline z)^2(E-1)^2K=0.
\]
The corresponding real basis contains
\[
  \cos\omega d,\quad \sin\omega d,\quad
  d\cos\omega d,\quad d\sin\omega d,\quad
  1,\quad d.
\]

\subsection{General PJ position space as a difference module}

Let
\[
  P(t)=\prod_{a=1}^{M}(t-z_a)^{m_a}.
\]
The finite-dimensional solution space of the constant-coefficient difference
equation
\[
  P(E)K=0
\]
is the exponential-polynomial space
\[
  K(d)
  =
  \sum_{a=1}^{M}
  \sum_{r=0}^{m_a-1}
  c_{a,r}\binom{d}{r}z_a^d.
\]
Thus the Fourier--jet part of \pjrope{} can be read as a finite
exponential-polynomial solution space generated by simple and repeated roots of
the shift operator.  The affine recency branch is the repeated root at
\(z=1\).

\subsection{Equivalence with finite-dimensional representations}
\label{app:diff_finite_representation}

The same structure appears as matrix coefficients of finite-dimensional
representations.  Suppose
\[
  K(d)=u^\top T^d v
\]
and
\[
  T\simeq \bigoplus_a z_a(I+N_a),
  \qquad N_a^{m_a}=0.
\]
Then
\[
  T^d
  \simeq
  \bigoplus_a
  z_a^d(I+N_a)^d
  =
  \bigoplus_a
  z_a^d
  \sum_{r=0}^{m_a-1}\binom{d}{r}N_a^r.
\]
Every matrix coefficient is therefore a finite linear combination of
\(\binom{d}{r}z_a^d\).

\begin{table}[H]
\centering
\caption{Three equivalent languages for the same finite-jet structure.}
\label{tab:difference-representation-dictionary}
\small
\begin{tabularx}{\linewidth}{>{\raggedright\arraybackslash}X>{\raggedright\arraybackslash}X>{\raggedright\arraybackslash}X}
\toprule
Difference equation & Representation theory & Spectral geometry \\
\midrule
simple root \(z\) & ordinary eigenvalue & reduced point \\
repeated root \(z\) & Jordan block & fat point \\
repeated root \(z=1\) & unipotent block & zero-frequency affine jet \\
\bottomrule
\end{tabularx}
\end{table}

\subsection{Stability and correlation decay from root location}
\label{app:diff_root_stability}

Write
\[
  z=\rho e^{i\omega}.
\]
The modulus \(\rho=|z|\) controls the long-distance behavior of the
corresponding lag mode:
\[
  z^d=\rho^d e^{i\omega d}.
\]
If \(0<\rho<1\), the mode decays exponentially,
\[
  \rho^d=e^{-d/\xi},
  \qquad
  \xi=-\frac{1}{\log\rho}.
\]
Thus roots inside the unit disk encode finite correlation length.  Roots
closer to the unit circle have longer memory.

The basic regimes are
\[
\begin{array}{lll}
|z|<1        &:& \text{exponentially decaying oscillation},\\
|z|=1,\ \text{simple root}
             &:& \text{bounded oscillation},\\
|z|=1,\ \text{repeated root}
             &:& \text{polynomially modulated oscillation},\\
|z|>1        &:& \text{exponential growth}.
\end{array}
\]
Repeated roots on or near the unit circle are exactly where polynomial growth
enters.  This explains why high-order Fourier jets are expressive but require
stabilization at long context.

LC-PJ is not itself an exact constant-coefficient difference-module solution,
because
\[
  \phi_L(d)=L\,\operatorname{asinh}(d/L)
\]
is a nonlinear coordinate substitution.  It should instead be understood as a
compactified deformation of the repeated-root jet coordinates, designed to
preserve local jet behavior while bounding far-field growth.

\paragraph{Fourier-distribution interpretation.}
There is an equivalent Fourier-distribution way to say the same thing.  The
ordinary Fourier character \(e^{i\omega d}\) corresponds to a point mass
\(\delta_\omega\) in frequency.  Since
\[
  d e^{i\omega d}
  =
  \frac{1}{i}\partial_\omega e^{i\omega d},
\]
the first Fourier jet corresponds to a derivative-of-delta direction at the
same frequency.  ALiBi's linear term is the analogous derivative direction at
zero frequency.  This is the distributional version of the repeated-root
statement: Jordan-RoPE replaces a spectral point mass by its finite jet, and
ALiBi is the zero-frequency affine jet.

\section{NTK-aware RoPE as a Spectral Flow of Difference Modules}
\label{app:ntk_diff_module}

This appendix relates NTK-aware RoPE scaling to the difference-module and
Fourier-jet viewpoint used in Appendix~\ref{app:difference-modules}.  We use
``NTK-aware'' in the sense of the community scaled-RoPE proposal and subsequent
base-frequency / long-context RoPE analyses
\cite{Bloc97NTK2023,Xiong2023LongContextScaling,Peng2024,Liu2024ScalingRoPE}.
The claim here is structural: NTK-aware RoPE moves the RoPE frequency grid, and
the infinitesimal tangent directions of this spectral motion are Fourier jets
in the lag-shift difference module.

\subsection{NTK-aware RoPE as spectral-root flow}

Using the notation of Appendix~\ref{app:difference-modules}, a RoPE character
\(\chi_z(d)=z^d\), \(z=e^{i\omega}\), is a simple root mode satisfying
\((E-z)\chi_z=0\).  NTK-aware RoPE scaling changes the effective frequency
grid used by RoPE.  Abstractly, this can be written as a one-parameter
deformation
\[
        z_\ell=z_\ell(\tau),
        \qquad
        \chi_{\ell,\tau}(d)=z_\ell(\tau)^d,
\]
where \(\tau\) is a scale parameter.  For each fixed \(\tau\),
\[
        (E-z_\ell(\tau))\chi_{\ell,\tau}=0.
\]
Thus NTK-aware RoPE may be viewed as a flow of simple roots in the spectral
base of the shift module.

Differentiating the root equation at \(\tau=0\) gives
\[
        (E-z_\ell)\dot{\chi}_{\ell}
        =
        \dot{z}_\ell \chi_{\ell},
\]
where \(z_\ell=z_\ell(0)\) and
\(\dot{\chi}_{\ell}=\partial_\tau\chi_{\ell,\tau}|_{\tau=0}\).
Applying \(E-z_\ell\) once more gives
\[
        (E-z_\ell)^2\dot{\chi}_{\ell}=0.
\]
Moreover,
\[
        \dot{\chi}_{\ell}(d)
        =
        \frac{\dot z_\ell}{z_\ell}d z_\ell^d.
\]
Therefore the tangent direction of an NTK-aware spectral flow is a first
generalized eigenvector of the lag shift.  In other words, it is a first
Fourier jet.

\subsection{Higher derivatives and repeated-root modules}

The same argument extends to higher derivatives.  Since
\[
        \partial_\tau^r z_\ell(\tau)^d\big|_{\tau=0}
\]
is a polynomial in \(d\) of degree at most \(r\) multiplied by \(z_\ell^d\), it
lies in the repeated-root module
\[
        \ker(E-z_\ell)^{r+1}.
\]
Equivalently, the first \(R\) derivatives of a spectral flow span a finite
Fourier-jet module
\[
        \mathcal{J}^{R}_{z_\ell}
        =
        \operatorname{span}
        \left\{
        z_\ell^d,
        dz_\ell^d,
        d^2z_\ell^d,
        \ldots,
        d^Rz_\ell^d
        \right\}.
\]
This is the repeated-root module
\[
        \mathbb{C}[E]/(E-z_\ell)^{R+1}
\]
up to the usual triangular change of basis between monomials \(d^r\) and
binomial coefficients \(\binom{d}{r}\).

Thus high-order PJ coordinates are a finite jet closure of the tangent
directions generated by NTK-aware RoPE scaling.

\subsection{Relation to the positional NTK}

Let \(K_\theta(d)\) be a scalar positional kernel.  The positional tangent
feature map is
\[
        \Psi_\theta(d)=\nabla_\theta K_\theta(d),
\]
and the corresponding positional tangent kernel is
\[
        \Theta^{\mathrm{pos}}_\theta(d,d')
        =
        \left\langle
        \nabla_\theta K_\theta(d),
        \nabla_\theta K_\theta(d')
        \right\rangle .
\]
When \(\theta\) contains a RoPE scale or frequency-flow parameter \(\tau\), the
feature
\[
        \partial_\tau e^{i\omega_\ell(\tau)d}\big|_{\tau=0}
\]
is proportional to
\[
        d e^{i\omega_\ell d}.
\]
Therefore the first-order positional NTK associated with an NTK-aware RoPE flow
contains first Fourier-jet directions.  If the parameterization explicitly
contains amplitudes for higher PJ orders,
\[
        K(d)
        =
        \sum_{\ell,r}
        a_{\ell,r}
        \left(\frac{d}{L}\right)^r
        e^{i\omega_\ell d},
\]
then the first-order tangent features with respect to \(a_{\ell,r}\) are
exactly the Fourier-jet basis functions
\[
        \partial_{a_{\ell,r}}K(d)
        =
        \left(\frac{d}{L}\right)^r
        e^{i\omega_\ell d}.
\]
Thus there are two related roles for jets.  NTK-aware RoPE produces first-order
jet directions by moving the frequency grid.  PJ-RoPE makes a finite family of
these jet directions explicit and learnable.

\subsection{Affine recency as the zero-frequency jet}

As Appendix~\ref{app:difference-modules} notes, the affine branch is the
repeated unit-root sector \(z=1\).  Equivalently,
\[
        d=\partial_\lambda e^{\lambda d}\big|_{\lambda=0}.
\]
Thus NTK-aware RoPE moves nonzero Fourier roots, while ALiBi-like recency
supplies the corresponding zero-frequency tangent direction.

\subsection{LC coordinates as stabilized jet coordinates}

The stability discussion in Appendix~\ref{app:difference-modules} explains why
raw high-order jets require a long-context chart.  LC-PJ replaces the raw
coordinate by
\[
        \phi_L(d)=L\,\operatorname{asinh}(d/L),
        \qquad
        \beta_L(d)=\frac{d}{\sqrt{d^2+L^2}},
\]
and uses compactified jet features of the schematic form
\[
        \beta_L(d)^r e^{i\omega \phi_L(d)}.
\]
This is a stabilized chart on the Fourier-jet module, not another exact
constant-coefficient difference-module solution.

\subsection{Summary}

The three views are compatible:
\[
\begin{array}{lll}
\text{Difference module} &:&
        \text{RoPE is a simple root; high-order jets are repeated roots.}
\\[2mm]
\text{NTK-aware RoPE} &:&
        \text{RoPE roots move along a spectral scaling flow.}
\\[2mm]
\text{PJ-RoPE} &:&
        \text{the tangent and higher tangent directions of this flow are made explicit.}
\end{array}
\]

In this language, NTK-aware RoPE is a low-dimensional spectral-flow baseline,
whereas PJ-RoPE is a finite Fourier-jet enlargement of the same local geometry.

\bibliographystyle{unsrturl}
\bibliography{references}

@inproceedings{Vaswani2017,
  author = {Vaswani, Ashish and Shazeer, Noam and Parmar, Niki and Uszkoreit, Jakob and Jones, Llion and Gomez, Aidan N. and Kaiser, Lukasz and Polosukhin, Illia},
  title = {Attention Is All You Need},
  booktitle = {Advances in Neural Information Processing Systems},
  volume = {30},
  pages = {5998--6008},
  year = {2017},
  url = {https://papers.nips.cc/paper_files/paper/2017/hash/3f5ee243547dee91fbd053c1c4a845aa-Abstract.html}
}

@inproceedings{Shaw2018,
  author = {Shaw, Peter and Uszkoreit, Jakob and Vaswani, Ashish},
  title = {Self-Attention with Relative Position Representations},
  booktitle = {Proceedings of the 2018 Conference of the North American Chapter of the Association for Computational Linguistics: Human Language Technologies, Volume 2 (Short Papers)},
  pages = {464--468},
  year = {2018},
  month = jun,
  address = {New Orleans, Louisiana},
  publisher = {Association for Computational Linguistics},
  doi = {10.18653/v1/N18-2074},
  url = {https://aclanthology.org/N18-2074/}
}

@inproceedings{Dai2019,
  author = {Dai, Zihang and Yang, Zhilin and Yang, Yiming and Carbonell, Jaime and Le, Quoc and Salakhutdinov, Ruslan},
  title = {Transformer-{XL}: Attentive Language Models beyond a Fixed-Length Context},
  booktitle = {Proceedings of the 57th Annual Meeting of the Association for Computational Linguistics},
  pages = {2978--2988},
  year = {2019},
  month = jul,
  address = {Florence, Italy},
  publisher = {Association for Computational Linguistics},
  doi = {10.18653/v1/P19-1285},
  url = {https://aclanthology.org/P19-1285/}
}

@misc{Su2021,
  author = {Su, Jianlin and Lu, Yu and Pan, Shengfeng and Murtadha, Ahmed and Wen, Bo and Liu, Yunfeng},
  title = {{RoFormer}: Enhanced Transformer with Rotary Position Embedding},
  year = {2021},
  eprint = {2104.09864},
  archivePrefix = {arXiv},
  primaryClass = {cs.CL},
  note = {arXiv preprint arXiv:2104.09864},
  url = {https://arxiv.org/abs/2104.09864}
}

@inproceedings{Press2022,
  author = {Press, Ofir and Smith, Noah A. and Lewis, Mike},
  title = {Train Short, Test Long: Attention with Linear Biases Enables Input Length Extrapolation},
  booktitle = {International Conference on Learning Representations},
  year = {2022},
  url = {https://openreview.net/forum?id=R8sQPpGCv0}
}

@misc{Chen2023,
  author = {Chen, Shouyuan and Wong, Sherman and Chen, Liangjian and Tian, Yuandong},
  title = {Extending Context Window of Large Language Models via Positional Interpolation},
  year = {2023},
  eprint = {2306.15595},
  archivePrefix = {arXiv},
  primaryClass = {cs.CL},
  note = {arXiv preprint arXiv:2306.15595},
  url = {https://arxiv.org/abs/2306.15595}
}

@inproceedings{Peng2024,
  author = {Peng, Bowen and Quesnelle, Jeffrey and Fan, Honglu and Shippole, Enrico},
  title = {{YaRN}: Efficient Context Window Extension of Large Language Models},
  booktitle = {International Conference on Learning Representations},
  year = {2024},
  url = {https://openreview.net/forum?id=wHBfxhZu1u},
  note = {arXiv:2309.00071}
}

@inproceedings{Ding2024,
  author = {Ding, Yiran and Zhang, Li Lyna and Zhang, Chengruidong and Xu, Yuanyuan and Shang, Ning and Xu, Jiahang and Yang, Fan and Yang, Mao},
  title = {{LongRoPE}: Extending {LLM} Context Window Beyond 2 Million Tokens},
  booktitle = {Proceedings of the 41st International Conference on Machine Learning},
  pages = {11091--11104},
  year = {2024},
  volume = {235},
  series = {Proceedings of Machine Learning Research},
  publisher = {PMLR},
  url = {https://proceedings.mlr.press/v235/ding24i.html}
}

@misc{Bloc97NTK2023,
  author = {{bloc97}},
  title = {{NTK-Aware Scaled RoPE} allows {LLaMA} models to have extended (8k+) context size without any fine-tuning and minimal perplexity degradation},
  year = {2023},
  howpublished = {Reddit post, r/LocalLLaMA},
  url = {https://www.reddit.com/r/LocalLLaMA/comments/14lz7j5/ntkaware_scaled_rope_allows_llama_models_to_have/},
  note = {Community proposal for NTK-aware scaled RoPE}
}

@misc{Xiong2023LongContextScaling,
  author = {Xiong, Wenhan and Liu, Jingyu and Molybog, Igor and Zhang, Hejia and Bhargava, Prajjwal and Hou, Rui and Martin, Louis and Rungta, Rashi and Sankararaman, Karthik Abinav and Oguz, Barlas and Khabsa, Madian and Fang, Han and Mehdad, Yashar and Narang, Sharan and Malik, Kshitiz and Fan, Angela and Bhosale, Shruti and Edunov, Sergey and Lewis, Mike and Wang, Sinong and Ma, Hao},
  title = {Effective Long-Context Scaling of Foundation Models},
  year = {2023},
  eprint = {2309.16039},
  archivePrefix = {arXiv},
  primaryClass = {cs.CL},
  doi = {10.48550/arXiv.2309.16039},
  url = {https://arxiv.org/abs/2309.16039},
  note = {arXiv preprint arXiv:2309.16039}
}

@inproceedings{Liu2024ScalingRoPE,
  author = {Liu, Xiaoran and Yan, Hang and Zhang, Shuo and An, Chenxin and Qiu, Xipeng and Lin, Dahua},
  title = {Scaling Laws of {RoPE}-based Extrapolation},
  booktitle = {International Conference on Learning Representations},
  year = {2024},
  eprint = {2310.05209},
  archivePrefix = {arXiv},
  primaryClass = {cs.CL},
  doi = {10.48550/arXiv.2310.05209},
  url = {https://arxiv.org/abs/2310.05209}
}

@inproceedings{Sun2023,
  author = {Sun, Yutao and Dong, Li and Patra, Barun and Ma, Shuming and Huang, Shaohan and Benhaim, Alon and Chaudhary, Vishrav and Song, Xia and Wei, Furu},
  title = {A Length-Extrapolatable Transformer},
  booktitle = {Proceedings of the 61st Annual Meeting of the Association for Computational Linguistics (Volume 1: Long Papers)},
  pages = {14590--14604},
  year = {2023},
  month = jul,
  address = {Toronto, Canada},
  publisher = {Association for Computational Linguistics},
  doi = {10.18653/v1/2023.acl-long.816},
  url = {https://aclanthology.org/2023.acl-long.816/}
}

@inproceedings{Chi2022,
  author = {Chi, Ta-Chung and Fan, Ting-Han and Ramadge, Peter J. and Rudnicky, Alexander I.},
  title = {{KERPLE}: Kernelized Relative Positional Embedding for Length Extrapolation},
  booktitle = {Advances in Neural Information Processing Systems},
  volume = {35},
  pages = {8386--8399},
  year = {2022},
  url = {https://proceedings.neurips.cc/paper_files/paper/2022/hash/37a413841a614b5414b333585e7613b8-Abstract-Conference.html}
}

@inproceedings{Li2024,
  author = {Li, Shanda and You, Chong and Guruganesh, Guru and Ainslie, Joshua and Ontanon, Santiago and Zaheer, Manzil and Sanghai, Sumit and Yang, Yiming and Kumar, Sanjiv and Bhojanapalli, Srinadh},
  title = {Functional Interpolation for Relative Positions Improves Long Context Transformers},
  booktitle = {International Conference on Learning Representations},
  year = {2024},
  url = {https://proceedings.iclr.cc/paper_files/paper/2024/hash/2f55a8b7b1c2c6312eb86557bb9a2bd5-Abstract-Conference.html},
  note = {arXiv:2310.04418}
}

@inproceedings{Haviv2022,
  author = {Haviv, Adi and Ram, Ori and Press, Ofir and Izsak, Peter and Levy, Omer},
  title = {Transformer Language Models without Positional Encodings Still Learn Positional Information},
  booktitle = {Findings of the Association for Computational Linguistics: EMNLP 2022},
  pages = {1382--1390},
  year = {2022},
  month = dec,
  address = {Abu Dhabi, United Arab Emirates},
  publisher = {Association for Computational Linguistics},
  doi = {10.18653/v1/2022.findings-emnlp.99},
  url = {https://aclanthology.org/2022.findings-emnlp.99/}
}

@inproceedings{Kazemnejad2023,
  author = {Kazemnejad, Amirhossein and Padhi, Inkit and Ramamurthy, Karthikeyan Natesan and Das, Payel and Reddy, Siva},
  title = {The Impact of Positional Encoding on Length Generalization in Transformers},
  booktitle = {Advances in Neural Information Processing Systems},
  volume = {36},
  year = {2023},
  url = {https://proceedings.neurips.cc/paper_files/paper/2023/hash/4e85362c02172c0c6567ce593122d31c-Abstract-Conference.html},
  note = {arXiv:2305.19466}
}

@inproceedings{Kogkalidis2024,
  author = {Kogkalidis, Konstantinos and Bernardy, Jean-Philippe and Garg, Vikas},
  title = {Algebraic Positional Encodings},
  booktitle = {Advances in Neural Information Processing Systems},
  volume = {37},
  year = {2024},
  doi = {10.52202/079017-1099},
  url = {https://proceedings.neurips.cc/paper_files/paper/2024/hash/3d8f2fdc04fa66c9239f2eb14379546d-Abstract-Conference.html}
}

@inproceedings{Ostmeier2024,
  author = {Ostmeier, Sophie and Axelrod, Brian and Varma, Maya and Moseley, Michael E. and Chaudhari, Akshay and Langlotz, Curtis},
  title = {{LieRE}: Lie Rotational Positional Encodings},
  booktitle = {Proceedings of the 42nd International Conference on Machine Learning},
  year = {2025},
  eprint = {2406.10322},
  archivePrefix = {arXiv},
  primaryClass = {cs.CV},
  note = {ICML 2025; arXiv:2406.10322},
  url = {https://arxiv.org/abs/2406.10322}
}

@inproceedings{Zhang2026GRAPE,
  author = {Zhang, Yifan and Chen, Zixiang and Liu, Yifeng and Qin, Zhen and Yuan, Huizhuo and Xu, Kangping and Yuan, Yang and Gu, Quanquan and Yao, Andrew Chi-Chih},
  title = {Group Representational Position Encoding},
  booktitle = {International Conference on Learning Representations},
  year = {2026},
  eprint = {2512.07805},
  archivePrefix = {arXiv},
  primaryClass = {cs.LG},
  doi = {10.48550/arXiv.2512.07805},
  url = {https://arxiv.org/abs/2512.07805},
  note = {ICLR 2026; arXiv:2512.07805}
}

@article{Raffel2020,
  author = {Raffel, Colin and Shazeer, Noam and Roberts, Adam and Lee, Katherine and Narang, Sharan and Matena, Michael and Zhou, Yanqi and Li, Wei and Liu, Peter J.},
  title = {Exploring the Limits of Transfer Learning with a Unified Text-to-Text Transformer},
  journal = {Journal of Machine Learning Research},
  volume = {21},
  number = {140},
  pages = {1--67},
  year = {2020},
  url = {https://jmlr.org/papers/v21/20-074.html}
}

@misc{Gao2024MEP,
  author = {Gao, Weiguo},
  title = {{MEP}: Multiple Kernel Learning Enhancing Relative Positional Encoding Length Extrapolation},
  year = {2024},
  eprint = {2403.17698},
  archivePrefix = {arXiv},
  primaryClass = {cs.LG},
  doi = {10.48550/arXiv.2403.17698},
  url = {https://arxiv.org/abs/2403.17698},
  note = {arXiv preprint arXiv:2403.17698}
}

@misc{Angelotti2023HyPE,
  author = {Angelotti, Giorgio},
  title = {{HyPE}: Attention with Hyperbolic Biases for Relative Positional Encoding},
  year = {2023},
  eprint = {2310.19676},
  archivePrefix = {arXiv},
  primaryClass = {cs.LG},
  doi = {10.48550/arXiv.2310.19676},
  url = {https://arxiv.org/abs/2310.19676},
  note = {arXiv preprint arXiv:2310.19676}
}

@misc{Zubkov2022FastRPB,
  author = {Zubkov, Maksim and Gavrilov, Daniil},
  title = {{FastRPB}: a Scalable Relative Positional Encoding for Long Sequence Tasks},
  year = {2022},
  eprint = {2202.11364},
  archivePrefix = {arXiv},
  primaryClass = {cs.LG},
  doi = {10.48550/arXiv.2202.11364},
  url = {https://arxiv.org/abs/2202.11364},
  note = {arXiv preprint arXiv:2202.11364}
}

@inproceedings{Huang2019,
  author = {Huang, Cheng-Zhi Anna and Vaswani, Ashish and Uszkoreit, Jakob and Shazeer, Noam and Simon, Ian and Hawthorne, Curtis and Dai, Andrew M. and Hoffman, Matthew D. and Dinculescu, Monica and Eck, Douglas},
  title = {Music Transformer: Generating Music with Long-Term Structure},
  booktitle = {International Conference on Learning Representations},
  year = {2019},
  url = {https://openreview.net/forum?id=rJe4ShAcF7},
  note = {arXiv:1809.04281}
}

@inproceedings{Hawthorne2019,
  author = {Hawthorne, Curtis and Stasyuk, Andriy and Roberts, Adam and Simon, Ian and Huang, Cheng-Zhi Anna and Dieleman, Sander and Elsen, Erich and Engel, Jesse and Eck, Douglas},
  title = {Enabling Factorized Piano Music Modeling and Generation with the {MAESTRO} Dataset},
  booktitle = {International Conference on Learning Representations},
  year = {2019},
  url = {https://openreview.net/forum?id=r1lYRjC9F7},
  note = {arXiv:1810.12247}
}

@inproceedings{Thickstun2017,
  author = {Thickstun, John and Harchaoui, Za{\"i}d and Kakade, Sham},
  title = {Learning Features of Music From Scratch},
  booktitle = {International Conference on Learning Representations},
  year = {2017},
  url = {https://openreview.net/forum?id=rkFBJv9gg},
  note = {arXiv:1611.09827}
}

@book{Hall2015,
  author = {Hall, Brian C.},
  title = {Lie Groups, Lie Algebras, and Representations: An Elementary Introduction},
  series = {Graduate Texts in Mathematics},
  volume = {222},
  edition = {2},
  publisher = {Springer},
  address = {Cham},
  year = {2015},
  doi = {10.1007/978-3-319-13467-3},
  url = {https://link.springer.com/book/10.1007/978-3-319-13467-3}
}

@article{Wigner1939,
  author = {Wigner, Eugene P.},
  title = {On Unitary Representations of the Inhomogeneous Lorentz Group},
  journal = {Annals of Mathematics},
  volume = {40},
  number = {1},
  pages = {149--204},
  year = {1939},
  doi = {10.2307/1968551},
  url = {https://inspirehep.net/literature/26312}
}

@misc{JordanRoPE2026,
  author = {Zhang, Yaobo},
  title = {{Jordan-RoPE}: Non-Semisimple Relative Positional Encoding via Complex Jordan Blocks},
  year = {2026},
  eprint = {2605.04217},
  archivePrefix = {arXiv},
  primaryClass = {cs.LG},
  doi = {10.48550/arXiv.2605.04217},
  url = {https://arxiv.org/abs/2605.04217},
  note = {arXiv preprint arXiv:2605.04217}
}

\end{document}